\documentclass[conference]{IEEEtran} 
\IEEEoverridecommandlockouts
\usepackage[american]{babel}

\usepackage{amssymb,latexsym,amsfonts,amsmath,amsthm,mathrsfs,verbatim}
\usepackage{bbm}
\usepackage{stfloats}
\usepackage{multicol}
\usepackage{multirow}
\usepackage{graphicx}
\graphicspath{ {./images/} }
\usepackage{url}
\usepackage{caption}
\usepackage{textcomp}
\usepackage{xcolor}
\usepackage{dsfont}
\usepackage{algorithm,algorithmic}
\usepackage{caption}
\usepackage{subcaption}
\usepackage{mathtools} 
\usepackage{booktabs} 
\usepackage{tikz} 
\newtheorem{theorem}{Theorem}
\newtheorem{lemma}{Lemma}

\newtheorem{assumption}{Assumption}
\newtheorem{remark}{Remark}

\newcommand{\Py}{{\mathbb P}}
\newcommand{\E}{{\mathbb E}}

\newcommand{\U}{{\mathcal U}}

\newcommand{\ltlf}{\textsc{LTL}_f}

\newcommand{\supp}{\mathrm{supp}}
\newcommand{\loc}[1]{{#1}^{\mathrm{loc}}}
\newcommand{\loci}[2]{{#1}^{\mathrm{loc,#2}}}
\newcommand\inlineeqno{\stepcounter{equation}\ (\theequation)}

\renewcommand{\baselinestretch}{1.00}

\title{
Optimal Control of Logically Constrained \\ Partially Observable and Multi-Agent \\ Markov Decision Processes
}
\author{Krishna C. Kalagarla$^1$, Dhruva Kartik$^2$, Dongming Shen$^2$, Rahul Jain$^3$, Ashutosh Nayyar$^3$ and Pierluigi Nuzzo$^3$
\thanks{$^{1}$K. Kalagarla was with the Ming Hsieh Department of Electrical and Computer Engineering, University of Southern California, Los Angeles, USA. He is now with the Electrical and Computer Engineering Department, University of New Mexico, Albuquerque, USA,
{\tt\small kalagarl@unm.edu}.}
\thanks{$^{2}$D. Kartik and D. Shen were with the Ming Hsieh Department of Electrical and Computer Engineering, University of Southern California, Los Angeles, USA, {\tt \small (mokhasun,alvinshe)@usc.edu.}}

\thanks{$^{3}$R. Jain, A. Nayyar and P. Nuzzo are with the Ming Hsieh Department of Electrical and Computer Engineering, University of Southern California, Los Angeles, USA, {\tt\small \{rahul.jain,ashutosh.nayyar, nuzzo\}@usc.edu}.}}
  
\begin{document}
\maketitle
\thispagestyle{plain}
\pagestyle{plain}
\begin{abstract} Autonomous systems often have logical constraints arising, for example, from safety, operational, or regulatory requirements. Such constraints can be expressed using temporal logic specifications. The system state is often partially observable. Moreover, it could encompass a team of multiple agents with a common objective but disparate information structures and constraints. In this paper, we first introduce an optimal control theory for partially observable Markov decision processes (POMDPs) with finite linear temporal logic ($\boldsymbol{\ltlf}$) constraints. We provide a structured methodology for synthesizing policies that maximize a cumulative reward while ensuring that the probability of satisfying a temporal logic constraint is sufficiently high.  Our approach comes with guarantees on approximate reward optimality and constraint satisfaction. We then build on this approach to design an optimal control framework for logically constrained multi-agent settings with information asymmetry. We illustrate the effectiveness of our approach by implementing it on several case studies.
\end{abstract}

\section{Introduction}\label{sec:intro}

Autonomous systems are rapidly being deployed in many safety-critical applications like robotics, transportation, and advanced manufacturing. Markov decision processes (MDPs)~\cite{Puterman:1994:MDP:528623} can model a wide range of sequential decision-making scenarios in these dynamically evolving environments. Traditionally, a reward structure is defined over the MDP state-action space, and is then maximized to achieve a desired objective. Formulating an appropriate reward function is critical, as an incorrect formulation can easily lead to unsafe and unforeseen behaviors. Designing reward functions for complex specifications can be exceedingly difficult and may not always be possible.

Increasing interest has been directed over the past decade toward leveraging tools from formal methods and temporal logic~\cite{baier2008principles} to alleviate this difficulty. These tools allow unambiguously specifying, solving, and validating complex control and planning problems. Temporal logic formalisms are capable of capturing a wide range of task specifications, including surveillance, reachability, safety, and sequentiality. However, while certain objectives like safety or reachability  are well expressed by temporal logic constraints, others, e.g., pertaining to system performance or cost, are often better framed as ``soft'' rewards to be maximized. In this paper, we focus on such composite tasks. 

Consider, for example, an autonomous robot tasked with navigating through a warehouse with hazardous areas to perform inspections or repairs. We can express complex requirements for this robot such as ``Always avoid hazardous areas," ``Eventually perform inspections," or ``Always return to a charging station when the battery is low," unambiguously via temporal logic. Model checking methods can be used to verify if a given controller satisfies these requirements~\cite{baier2008principles}. We can also use algorithmic methods to synthesize a controller for the robot such that the resulting controller satisfies the requirements by construction~\cite{baier2008principles}. This is the approach we adopt in the paper. Further, we wish to account for additional considerations, such as fuel efficiency or smoothness of motion, expressed by reward functions to be maximized.

While full state observability is assumed in environments modeled by MDPs, this assumption excludes many real-life scenarios where the state is only partially observed. These scenarios can instead be captured by partially observable Markov decision processes (POMDPs). For example, the warehouse robot only has a partial observation of the state since the sensor  observations only provide a noisy estimate of it.

In this paper, we expand the traditional POMDP framework to incorporate temporal logic specifications. Specifically, we aim to synthesize policies such that the agent's cumulative reward is maximized while the probability of satisfying a given temporal logic specification is above a desired threshold. 
Due to the inherent partial observability, planning methods for MDPs with temporal logic specifications~\cite{kalagarla2021optimal,sickert2016limit,hahn2019omega} are not directly applicable to our problem. Recasting our problem as an MDP would necessitate constructing beliefs over an augmented state space, which grows exponentially with the time horizon. This results in an extremely large belief space, making the synthesis methods developed for MDPs intractable in the context of POMDPs. Consequently, we utilize scalable algorithms specifically designed for POMDPs to solve our problem effectively.

We consider finite linear temporal logic ($\ltlf$)~\cite{de2013linear}, a temporal extension of propositional logic, to express complex tasks. $\ltlf$ is a variant of linear temporal logic (LTL)~\cite{baier2008principles}, which is interpreted over finite length strings rather than infinite length strings. 
 We can express the requirements of our warehouse robot example in $\ltlf$, respectively, as follows: $Always~\neg \text{hazard.nearby}, \enspace Eventually \enspace \text{inspection.done}, \\Always \enspace   (\text{low.battery} \to Eventually \enspace \text{charging.station}) $. We need to simultaneously consider the cumulative POMDP reward and the temporal logic satisfaction. For a given $\ltlf$ specification, we construct a deterministic finite automaton (DFA) that accepts an agent's trajectory if and only if it satisfies the specification~\cite{zhu2017symbolic}. By augmenting the system state with the DFA's internal state, we can monitor both environmental and task status. This approach enables us to frame our planning problem as a standard reward-constrained POMDP problem. In the context of our motivating example, we combine the POMDP expressing the robot dynamics in the warehouse and the  DFA expressing our  specifications into a constrained POMDP problem containing details of both the POMDP and the DFA (thus the $\ltlf$ specification). 
 
We then tackle the challenge of constrained POMDP problems by proposing an iterative primal-dual scheme which solves a sequence of unconstrained POMDP problems using any off-the-shelf unconstrained POMDP solver~\cite{kurniawati2008sarsop,silver2010monte}. This approach effectively leverages the established techniques of unconstrained POMDP planning. The iterative scheme also incorporates regret bounds from no-regret online learning \cite{hazan2016introduction}, providing guarantees on the near-optimality of the returned policy.

Finally, we extend our framework to a multi-agent setting with information asymmetry. We use the previously described approach for composition with the DFA to obtain a constrained multi-agent planning problem. In this case, our approach requires solving a sequence of unconstrained multi-agent problems. Under some mild assumptions, these unconstrained multi-agent problems can be viewed as single-agent POMDP problems via the \emph{common information approach}~\cite{nayyar2013decentralized}. Owing to the theoretical guarantees of the common information approach and our algorithm, the returned policies will still provide the desired guarantees on optimality and temporal logic satisfaction.

Our contributions can be summarized as follows: 
   (i) We formulate a novel optimal control problem in terms of cumulative reward maximization in POMDPs under expressive $\ltlf$ constraints.
   (ii) We design an iterative planning scheme which can leverage any off-the-shelf unconstrained POMDP solver to solve the problem. Differently from existing work on constrained POMDPs, our scheme uses a no-regret online learning approach to provide theoretical guarantees on the near-optimality of the returned policy. 
   (iii) We extend the single-agent framework to a multi-agent setting with information asymmetry, where agents have a common reward objective but can have individual or even joint temporal logic requirements. Their different information structure makes the problem rather difficult. 
   (iv) We conduct experimental studies on a suite of single and multi-agent environments to validate the algorithms developed and explore their effectiveness. To the best of our knowledge, this is the first effort on optimal control of logically constrained partially observed and multi-agent MDP models.

We presented some preliminary results on optimal control of POMDPs with $\ltlf$ specifications in a previous paper~\cite{kalagarla2022optimal}. In this paper, we expand on our previous results and show that our methodology is broader in scope, in that it encompasses other problem formulations and solution strategies. 
In particular, we extend our planning algorithm from a single-agent setting 
to a \emph{multi-agent setting with information asymmetry}. Moreover, while we previously assumed access to an \emph{exact} POMDP solver and an exact estimator of the expected total reward, we show that our method can also operate with an approximate POMDP solver and approximate estimates of the expected total reward. We then show that our method also supports cumulative reward constraints along with $\ltlf$ requirements. Finally, we include all the proofs of our theorems.

\section{Related Work}

Synthesizing policies for MDPs such that they maximize the probability of satisfying a temporal logic specification has been extensively studied in the context of known \cite{ding2011ltl,sickert2016limit} and unknown \cite{hahn2019omega,bozkurt2019control} transition probabilities. 

Several approaches have also been proposed for maximizing the  temporal logic satisfaction probability in POMDPs, albeit without considering an additional reward objective. Some of these methods include building a finite-state abstraction via approximate stochastic simulation over the belief space~\cite{haesaert2018temporal}, grid-based discretization of the POMDP belief space~\cite{norman2015verification},  
 and restricting the space of policies to finite state controllers~\cite{ahmadi2020stochastic, sharan2014finite}. Similarly to our work, leveraging well-studied unconstrained POMDP planners~\cite{liuleveraging,bouton2020point} has also been proposed in this context, while  deep learning approaches like those using recurrent neural networks~\cite{carr2020verifiable,carr2019counterexample} have been lately explored.

Our approach further differs from a few other existing methods for solving constrained POMDPs. 
One of these methods \cite{poupart2015approximate} addresses a constrained POMDP by modeling it as a constrained belief MDP and employing an approximate linear program that operates on a selectively refined subset of reachable beliefs. However, the expanding size of the subset of reachable beliefs and the resulting increase in the complexity of the linear programs eventually render this approach computationally intractable.
Another method addresses the scalability issues of the previous one using a primal-dual approach based on Monte Carlo Tree Search (MCTS) \cite{lee2018monte}. We solve the constrained POMDP problem using a similar method.  A key difference is that, instead of using an MCTS approach, we use an approximate unconstrained POMDP solver, SARSOP \cite{kurniawati2008sarsop},  which returns policies along with bounds on their optimality gaps.  
Column generation algorithms \cite{walraven2018column} also use a primal-dual approach, but with a different dual parameter update procedure. While only convergence to optimality is discussed for these algorithms, our method, on the other hand, gives a precise relationship between the approximation error and the number of iterations.

A few methods have also appeared which address temporal logic satisfaction in a multi-agent setting, albeit leveraging different formulations and solution strategies. Some examples include the synthesis of a joint policy maximizing the temporal logic satisfaction probability for multiple agents with shared state space \cite{hammond2021multi} and the decomposition of temporal logic specifications for scalable planning for a large number of agents under communication constraints \cite{wang2022decentralized,eappen2022distspectrl}. Some efforts have also leveraged the robustness semantics of temporal logics to guide reward shaping for multi-agent planning~\cite{zhang2022distributed,sun2020automata}. To the best of our knowledge, our method is the first to consider a reward objective and a temporal logic constraint in the multi-agent stochastic setting with information asymmetry.

\section{Preliminaries} \label{sec:prelim}

We denote the POMDP by $\mathscr{M}$ and the $\ltlf$ specification by $\varphi$. The DFA equivalent to the $\ltlf$ specification is symbolized by $\mathscr{A}$ and the product POMDP constructed with $\mathscr{M}$ and $\mathscr{A}$ is designated as $\mathscr{M} ^{\times}$. $\Py_{\varphi}^{\mathscr{M}}(\mu)$ denotes the probability that a run of $\mathscr{M}$ satisfies $\varphi$ under policy $\mu$. The indicator function $\mathds{1}_{S}(s)$ evaluates to $1$ when $s\in S$ and 0 otherwise. 
The probability simplex over the set $S$ is denoted by $\Delta_{S}$.  Finally, for a string $s$, $|s|$ denotes its length.

\subsection{Labeled POMDPs} \label{sec:prelim_a}
\subsubsection{Model} 
A Labeled Partially Observable Markov Decision Process is defined as a tuple $ \mathscr{M} = ({\mathcal{S}},{\U},P,\varpi,{O},Z,A P,L,r^o,r^c,T)$, where ${\mathcal{S}}$ is a finite state space, ${\U}$ is a finite action space, $P_t(s,u;s')$ is the probability of transitioning from state $s$ to state $s'$ on taking action $u$ at time $t$, $\varpi \in \Delta_{\mathcal{S}}$ is the initial state distribution, ${O}$ is a finite observation space, $Z_t(s;o)$ is the probability of seeing observation $o$ in state $s$ at time $t$, $AP$ is a set of atomic propositions, used to indicate the truth value of a predicate (property) of the state, e.g., the presence of an obstacle or goal state. $L: {\mathcal{S}} \to 2^{AP}$ is a labeling function which indicates the set of atomic propositions which are true in each state, e.g., $L(s) = \{a\}$ indicates that only the atomic proposition $a$ is true in state $s$. $r^{o}_{t}(s,u)$ and $r^{c}_{t}(s,u)$ are the objective and constraint rewards obtained on taking action $u$ in state $s$ at time $t$. The state, action, and observation at time $t$ are denoted by $\mathcal{S}_t$, $\U_{t}$, and $O_t$, respectively. At any given time $t$, the information available to the agent is the collection of all the observations $O_{0:t}$ and all the past actions $\U_{0:t-1}$. We denote this information with $I_t = \{O_{0:t},\U_{0:t-1}\}$. We say that the system is time-invariant when the reward functions $r^o_t, r^c_t$ and the transition and observation probability functions $P_t$ and $Z_t$ do not depend on time $t$. The POMDP runs for a random time horizon $T$. This random time may be determined exogenously (independently) of the POMDP or it may be a stopping time with respect to the information process $\{I_t: t \geq 0\}$.

\subsubsection{Pure and Mixed Policies}

A \emph{control law} $\pi_t$ maps the information $I_t$ to an action in the action space $\U$. A \textit{policy} $\pi := (\pi_{0},\pi_1,\ldots)$ is then the sequence of laws over the entire horizon. We refer to such deterministic policies as pure policies and denote the set of all pure policies with $\mathcal{P}$. 

A mixed policy $\mu$ is a distribution on a finite collection of pure policies. Under a mixed policy $\mu$, the agent randomly selects a pure policy $\pi \in \mathcal{P}$ with probability $\mu(\pi)$ before the POMDP begins. The agent uses this randomly selected policy to select its actions during the course of the process. More formally, $\mu:\mathcal{P}\to [0,1]$ is a mapping. The support of the mixture $\mu$ is defined as $ \supp(\mu):=\{\pi \in \mathcal{P} :\mu(\pi)\neq 0 \} \quad \inlineeqno$

If the support size of a mixed policy is $1$, then it is a pure policy.
The set $\mathcal{M}_p$ of all mixed mappings is given by
\begin{align}
    \mathcal{M}_p:=\left\{\mu:|\supp(\mu)|<\infty ,\sum_{\pi \in \supp(\mu)}\mu(\pi)=1\right\}.\label{mixedpol}
\end{align}
Clearly, the set $\mathcal{M}_p$ of mixed policies is convex. 

A \emph{run} ${\xi}$ of the POMDP is the sequence of  states and actions $((\mathcal{S}_0, \U_0), (\mathcal{S}_1,\U_1), \cdots ,(\mathcal{S}_{T},\U_{T}))$ up to horizon $T$. The total expected objective reward for a policy $\mu$ is given by
\begin{align}
    \mathcal{R}^{\mathscr{M}}_{o}(\mu) &= \E_\mu^\mathscr{M}\left[\sum_{t=0}^T r^{o}_{t}(\mathcal{S}_t,\U_t)\right]\\
    &=\sum_{\pi \in \supp(\mu)}\left[\mu(\pi)\E_\pi^\mathscr{M}\left[\sum_{t=0}^T r^{o}_{t}(\mathcal{S}_t,\U_t)\right]\right].
\end{align}
The total expected constraint reward  $\mathcal{R}^{\mathscr{M}}_{c}(\mu)$ is similarly defined with respect to $r^{c}$. $\mathcal{R}^{\mathscr{M}}_{o}(\mu)$ and $\mathcal{R}^{\mathscr{M}}_{c}(\mu)$  are linear functions in $\mu$.

\begin{assumption}\label{finiteassumpstrong}
The POMDP $\mathscr{M}$ is such that for every pure policy $\pi$, the expected value of the random horizon $T$ is bounded by a finite constant and the total expected constraint reward is also bounded above by a finite constant, i.e.,
\begin{align}
    &\E_\pi^{\mathscr{M}}[T] < T_{\textsc{max}} < \infty, \\
    &0 \le \mathcal{R}^{\mathscr{M}}_{c}(\pi)\le \mathbf{R}^{max}_{c} ~~\forall \ \pi.
\end{align}
The expectation in $\E_\pi^{\mathscr{M}}[T]$ is over the random state, action, and observation trajectory of POMDP $\mathscr{M}$ under the pure policy $\pi$.
\end{assumption}

Assumption \ref{finiteassumpstrong} ensures that  $T$ is finite almost surely, i.e., $\Py_\mu^\mathscr{M}[T<\infty]=1$ and the total expected constraint reward satisfies $\mathcal{R}^{\mathscr{M}}_{c}(\mu) \leq \mathbf{R}^{max}_{c} $ for every policy $\mu$.

\subsection{Finite Linear Temporal Logic}

We use $\ltlf$~\cite{de2013linear}, a variant of linear temporal logic (LTL)~\cite{baier2008principles} interpreted over finite strings,  to express complex task specifications. Given a set $AP$ of atomic propositions, i.e., Boolean variables that have a unique truth value ($\mathsf{true}$ or $\mathsf{false}$) for a given system state, $\ltlf$ formulae are constructed inductively as  follows: 
\begin{equation*}
    \varphi := \mathsf{ true } \ | \ a \ | \ \neg  \varphi \ | \ \varphi_1 \wedge \varphi_2 \ | \ \textbf{X} \varphi \ | \ \varphi_1 \textbf{U} \varphi_2, 
\end{equation*}
where $a \in AP$, $\varphi$, $\varphi_1$, and $\varphi_2$ are LTL formulae, $\wedge$ and $\neg$ are the logic conjunction and negation,  and $\textbf{U}$ and $\textbf{X}$ are the \emph{until} and \emph{next} temporal operators. Additional temporal operators such as \emph{eventually} ($\textbf{F}$) and \emph{always} ($\textbf{G}$)  are derived as $\textbf{F} \varphi := \mathsf{ true } \textbf{U} \varphi$ and $\textbf{G} \varphi := \neg \textbf{F} \neg \varphi$. For example, $ \varphi = \textbf{F}a\wedge(\textbf{G}\neg b)$ expresses the specification that a state where atomic proposition $a$ holds true has to be 
\emph{eventually} reached by the end of the trajectory and states where atomic proposition $b$ holds true have to be \emph{always} avoided.

$\ltlf$ formulae are interpreted over finite-length words $w = w_0w_1 \cdots w_{last} \in {(2^{AP})}^{*}$, where each letter $w_i$ is a set of atomic propositions and $last = |w| - 1$ is the index of the last letter of the word $w$. 
Given a finite word $w$ and $\ltlf$ formula $\varphi$, we inductively define when $\varphi$ is $\mathsf{true}$ for $w$ at step $i$, $0 \leq i < |w|$, written $(w,i) \models \varphi$. Informally, $a$ is $\mathsf{true}$ for $(w,i)$ iff $a \in w_i$; $\textbf{X} \varphi$ is $\mathsf{true}$ for $(w,i)$ iff 
$\varphi$ is $\mathsf{true}$ for $(w,i+1)$; $ \varphi_1 \textbf{U}  \varphi_2$ is $\mathsf{true}$ for $(w,i)$ iff there exits $k \geq i $ such that $\varphi_2$ is $\mathsf{true}$ for $(w,k)$ and $\varphi_1$ is $\mathsf{true}$ for all $j, \  i \leq  j  < k$; $\textbf{G} \varphi$ is $\mathsf{true}$ for $(w,i)$ iff $\varphi$ is $\mathsf{true}$ for all $(w,j), j\geq i$; $\textbf{F} \varphi$ is $\mathsf{true}$ for $(w,i)$ iff $\varphi$ is $\mathsf{true}$ for some $(w,j), j\geq i$ (where `iff' is shorthand for `if and only if'). A formula $\varphi$ is $\mathsf{true}$ in $w$, denoted by $w \models \varphi $, iff $(w,0) \models \varphi$.

Given a POMDP $ \mathscr{M}$ and an $\ltlf$ formula $\varphi$, a \emph{run} $\xi = ((s_0,u_0),(s_1,u_1),\cdots, (s_T,u_T))$ of the POMDP 
under policy $\mu$ is said to satisfy $\varphi$ if the
word $w = L(s_0)L(s_1)\cdots \in {(2^{AP})}^{T+1}$ generated by the \emph{run} satisfies $\varphi$. The probability that a run of $\mathscr{M}$ satisfies $\varphi$ under policy $\mu$ is denoted by $\Py_{\varphi}^{\mathscr{M}}(\mu)$. We refer to Section \ref{exp} for various examples of $\ltlf$ specifications, especially those which cannot be easily expressed by standard reward functions.

\subsection{Deterministic Finite Automaton (DFA)}

The language defined by an $\ltlf$ formula, i.e., the set of words satisfying the formula, can be captured by a Deterministic Finite Automaton (DFA)~\cite{zhu2017symbolic}. A DFA is a finite state machine  that accepts or rejects a given string of symbols, by running through a state sequence uniquely determined by the string. It consists of a finite set of states, a set of input symbols, a transition function, a start state, and a set of accepting states~\cite{sipser1996introduction}. We denote such a DFA by a tuple $\mathscr{A} = (Q, \Sigma, q_0, \delta, F)$, where $Q$ is a finite set of states, $\Sigma$ is a finite alphabet, $q_0$ is the initial state, $\delta :
Q \times \Sigma   \to  {Q}$  is a transition function, and $F \subseteq Q$ is the set of accepting states. A \emph{run} $\xi_{\mathscr{A}}$ of $\mathscr{A}$ over a finite word $w = w_0\cdots w_n$, with $w_i \in \Sigma$, is a  sequence of states, $q_0q_1\cdots q_{n+1} \in Q^{n+1}$ such that $q_{i+1} = \delta(q_i,w_i), i = 0,\cdots,n$. A \emph{run} $\xi_{\mathscr{A}}$ is \emph{accepting} if it ends in an accepting state, i.e., 
$q_{n+1} \in F$. A word $w \in \Sigma^{*}$ is accepted by $\mathcal{A}$ if and only if there exists an accepting \emph{run} $\xi_{\mathscr{A}}$ of $\mathcal{A}$ on $w$.
Finally, we say that an $\ltlf$ formula is equivalent to a DFA $\mathscr{A}$ if and only if the language defined by the formula is the language (i.e., the set of words) accepted by $\mathscr{A}$. For any $\ltlf$ formula $\varphi$ over $AP$, we can construct an equivalent DFA with input alphabet $2^{AP}$~\cite{zhu2017symbolic}.

\section{Problem Formulation and Solution Strategy}\label{sec:prob_for}

Given a labeled POMDP $\mathscr{M}$ and an $\ltlf$ specification $\varphi$, our objective is to design a policy $\mu$ that maximizes the total expected objective reward $\mathcal{R}^{\mathscr{M}}_{o}(\mu)$ while ensuring that the probability $\Py_{\varphi}^{\mathscr{M}}(\mu)$ of satisfying the specification $\varphi$ is at least $1-\delta$ and the total expected constraint reward $\mathcal{R}^{\mathscr{M}}_{c}(\mu)$ exceeds a threshold $\rho$. $\mathcal{R}^{\mathscr{M}}_{o}(\mu), \Py_{\varphi}^{\mathscr{M}}(\mu)$, and $\mathcal{R}^{\mathscr{M}}_{c}(\mu)$ are all used to express different requirements which we wish to satisfy. More formally, we would like to solve the following constrained optimization problem
\begin{equation}\tag{P1} \label{probform}
    \begin{aligned}
   \textbf{LTL$_f$-POMDP:}~~~\underset{\mu  \in \mathcal{M}_p}{\sup} \quad & {\mathcal R}^{\mathscr{M}}_{o}(\mu)\\ 
   \mathrm{s.t.} \quad     & \mathcal{R}^{\mathscr{M}}_{c}(\mu) \geq \rho,\\
& \Py_{\varphi}^{\mathscr{M}}(\mu) \geq  1-\delta.
\end{aligned}
\end{equation}
If \eqref{probform} is feasible, then we denote its optimal value with $\mathcal{R}^{\mathscr{M}}_{*}$. If \eqref{probform} is infeasible, then $\mathcal{R}^{\mathscr{M}}_{*} = -\infty$.

\begin{remark}
The results and experiments in this work are in the setting of  POMDPs and multi-agent MDPs with finite state, action, and observation spaces.
\end{remark}

\subsection{Constrained Product POMDP}\label{conspomdp}

In \eqref{probform}, we require a policy $\mu$ to maximize the objective reward and satisfy a reward constraint, both of which are expressed via the given POMDP $\mathscr{M}$. We also need the same policy $\mu$ to satisfy the $\ltlf$ constraint which is expressed via a DFA $\mathscr{A}$ (equivalent to the $\ltlf$ specification $\varphi$) instead of the POMDP $\mathscr{M}$. We thus construct a constrained product POMDP $\mathscr{M}^{\times}$ which allows us to address these different requirements via a single unified model. 

Given the labeled POMDP $ \mathscr{M}$ and a DFA $\mathscr{A}$ capturing the $\ltlf$ formula $\varphi$,
we follow a construction previously proposed for MDPs~\cite{kalagarla2022optimal} to obtain a constrained product POMDP $\mathscr{M} ^{\times} = (\mathcal{S}^{\times},  \U^{\times},P^{\times},s_{0}^{\times},r^{o\times},r^{c\times},r^f,\varpi,{O},Z^\times)$  which incorporates the transitions of $ \mathscr{M}$  and $ \mathscr{A}$, the observations and the reward functions of $\mathscr{M}$, and the acceptance set of $\mathscr{A}$.

In the constrained product POMDP $\mathscr{M} ^{\times}$, $\mathcal{S}^{\times} = ({\mathcal{S}} \times Q)$ is the state space, $\U^{\times} = {\U}$ is the action space, and $s_{0}^{\times} = (s_0,q_0)$ is the initial state where the POMDP's initial state $s_0$ is drawn from the distribution $\varpi$ and $q_0$ is the DFA's initial state. For each pair $(s,s')$, $(q,q')$, and $u$, we define the transition function $P^{\times}_t((s,q),u;(s',q'))$ at time $t$ as
\begin{equation} \label{eq:prodtrans}
\begin{aligned}
& \begin{cases} P_t(s,u;s'), &\mbox{if } q' = \delta(q,L(s)), \\
0, & \text{otherwise.}
\end{cases}\\
\end{aligned}    
\end{equation}
The reward functions are defined as 
\begin{align}
    r^{o\times}_{t}((s,q),u) &= {r}^{o}_t(s,u),~ \forall s,q,u \label{prodrewdefo}\\
    r^{c\times}_{t}((s,q),u) &= {r}^{c}_t(s,u),~ \forall s,q,u\label{prodrewdefc}\\
    \label{finalrewdefm}r^f((s,q)) &=
    \begin{cases}
    1,~ &\text{if }q \in F\\
    0,~ &\text{otherwise}.
    \end{cases}
\end{align}
The reward functions $r^{o\times}$ and $r^{c\times}$ depend on the reward functions of the original POMDP $\mathscr{M}$. The reward function $r^f$ is instead informed by the accepting states of the DFA $\mathscr{A}$.
The observation space ${O}$ is the same as in the original POMDP $\mathscr{M}$. The observation probability function $Z^\times((s,q);o)$ is defined to be equal to $Z(s;o)$ for every $s,q,o$. We denote the state of the product POMDP $\mathscr{M}^\times$ at time $t$ with $X_t = (\mathcal{S}_t,Q_t)$ in order to avoid confusion with the state $\mathcal{S}_t$ of the original POMDP $\mathscr{M}$.

At any given time $t$, the information available to the agent is $I_t = \{O_{0:t},\U_{0:t-1}\}$. Control laws and policies in the product POMDP are the same as in the original POMDP $\mathscr{M}$. We define three reward functions in the product POMDP: (i) a reward $\mathcal{R}^{\mathscr{M}^\times}_{o}(\mu)$ associated with the objective reward $r^{o}$, (ii) a reward $\mathcal{R}^{\mathscr{M}^\times}_{c}(\mu)$ associated with the constraint reward $r^{c}$, and (iii) a reward $\mathcal{R}^{f}(\mu)$ associated with reaching an accepting state in the DFA $\mathscr{A}$.
The expected reward $\mathcal{R}^{\mathscr{M}^\times}_{o}(\mu)$ is defined as
\begin{align}
    \mathcal{R}^{\mathscr{M}^\times}_{o}(\mu) = \E_\mu\left[\sum_{t=0}^T r_t^{o\times}(X_t,\U_t)\right].
\end{align}
The expected reward $\mathcal{R}^{\mathscr{M}^\times}_{c}(\mu)$ is similarly defined as:
\begin{align}
    \mathcal{R}^{\mathscr{M}^\times}_{c}(\mu) = \E_\mu\left[\sum_{t=0}^T r_t^{c\times}(X_t,\U_t)\right].
\end{align}
The expected reward $\mathcal{R}^{f}(\mu)$ is defined as
\begin{align}
    \mathcal{R}^{f}(\mu) &= \E_\mu\left[ r^f(X_{T+1})\right].\\
    &=\sum_{\pi \in \supp(\mu)}\mu(\pi)\E_\pi\left[r^f(X_{T+1})\right].\label{eq:linear}
\end{align}
Due to Assumption \ref{finiteassumpstrong}, the stopping time $T$ is finite almost surely and therefore, the reward $\mathcal{R}^{f}(\mu)$ is well-defined. Further, \eqref{eq:linear}, which follows  from the definition of mixed policies, demonstrates that $\mathcal{R}^{f}(\mu)$ is a linear function of $\mu$.

\subsection{Constrained POMDP Formulation}

In the  product POMDP $\mathscr{M}^\times$, we are interested in solving the following constrained optimization problem
\begin{equation} \tag{P2}\label{prodprobform}
    \begin{aligned}
   \textbf{C-POMDP:}~~~\underset{\mu  \in \mathcal{M}_p }{\sup} \quad & {\mathcal R}^{\mathscr{M}^\times}_{o}(\mu)\\ 
      \mathrm{s.t.} \quad  & {\mathcal R}^{\mathscr{M}^\times}_{c}(\mu) \geq \rho,\\
     & \mathcal{R}^{f}(\mu) \geq  1-\delta
\end{aligned}
\end{equation}
whose optimal value is denoted by $\mathcal{R}_{*}^{\mathscr{M}^\times}$.

\begin{theorem}[Equivalence of Problems \eqref{probform} and \eqref{prodprobform}]\label{equiv1}
For any policy $\mu$, we have
\begin{align*}
    \mathcal{R}^{\mathscr{M}^\times}_{o}(\mu) &= \mathcal{R}^{\mathscr{M}}_{o}(\mu),\\
    \mathcal{R}^{\mathscr{M}^\times}_{c}(\mu) = \mathcal{R}^{\mathscr{M}}_{c}(\mu)&,~\text{and}~
    \mathcal{R}^{f}(\mu) = \Py_{\varphi}^{\mathscr{M}}(\mu).
\end{align*}
Therefore, a policy $\mu^*$ is an optimal solution to Problem \eqref{probform} if and only if it is an optimal solution to Problem \eqref{prodprobform}, and therefore, $\mathcal{R}_{*}^{\mathscr{M}} = \mathcal{R}_{*}^{\mathscr{M}^\times} := \mathcal{R}^*$.
\end{theorem}

Proofs of all theorems and lemmas are available in the appendices.

\section{A No-regret Learning Approach for Solving the Constrained POMDP}\label{noreg}

Problem \eqref{prodprobform} is a POMDP policy optimization problem with constraints. Since solving unconstrained optimization problems is generally easier than solving constrained optimization problems, we describe a general methodology that reduces the constrained POMDP optimization problem \eqref{prodprobform} to a sequence of unconstrained POMDP problems. These unconstrained POMDP problems can be solved using any off-the-shelf POMDP solver. The main idea is to first transform Problem \eqref{prodprobform} into a sup-inf problem using the Lagrangian function. This sup-inf problem can then be solved approximately using a no-regret online learning algorithm such as the exponentiated gradient (EG) algorithm \cite{hazan2016introduction}. 

The Lagrangian function $L(\mu,\boldsymbol{\lambda})$ associated with Problem \eqref{prodprobform} with $\boldsymbol{\lambda} = (\lambda^f,\lambda^c)$ is
\begin{align*}
    \mathcal{R}^{\mathscr{M}^\times}_{o}(\mu) + \lambda^c(\mathcal{R}^{\mathscr{M}^\times}_{c}(\mu) - \rho) + \lambda^f (\mathcal{R}^{f}(\mu)-1+\delta).
\end{align*}
Let
\begin{align}
    l^*:=\sup_{\mu}\inf_{\boldsymbol{\lambda}\geq 0} L(\mu,\boldsymbol{\lambda})\tag{P3}\label{supinf}.
\end{align}
The constrained optimization problem in \eqref{prodprobform} is equivalent to the sup-inf optimization problem above \cite{boyd2004convex}. That is, if an optimal solution $\mu^*$ exists in problem \eqref{prodprobform}, then $\mu^*$ is a maximizer in \eqref{supinf}, and if $\eqref{prodprobform}$ is infeasible, then $l^*=-\infty$. Further, the optimal value of Problem \eqref{prodprobform} is equal to $l^*$. Consider the following variant of \eqref{supinf} wherein the Lagrange multiplier $\boldsymbol{\lambda}$ is bounded in the $L^1$ norm.
\begin{align}
    l^*_B:=\sup_{\mu}\inf_{\boldsymbol{\lambda}\geq 0, \lVert\boldsymbol{\lambda} \rVert_{1}\leq B} L(\mu,\boldsymbol{\lambda})\tag{P4}\label{bsupinf}.
\end{align}
\begin{lemma}\label{epsopt}
Let $\bar{\mu}$ be an $\epsilon$-optimal policy in sup-inf problem \eqref{bsupinf}, i.e.,
\begin{align}
    l^*_B \leq \inf_{\boldsymbol{\lambda}\geq 0,\lVert\boldsymbol{\lambda \rVert}_{1}\leq B}L(\bar{\mu},\boldsymbol{\lambda}) + \epsilon,\label{epsoptcond}
\end{align}
for some $\epsilon>0$. Then, we have
\begin{align}
    \mathcal{R}^{\mathscr{M}^\times}_{o}(\bar{\mu}) &\geq \mathcal{R}^*-\epsilon,\label{objrewsat}\\
    \mathcal{R}^{\mathscr{M}^\times}_{c}(\bar{\mu}) &\geq \rho - \epsilon^f ~~\text{and}\label{conrewsat}\\
   \mathcal{R}^{f}(\bar{\mu})&\geq 1-\delta - \epsilon^f,\label{consat}
\end{align}
where $\epsilon^f = \frac{\mathbf{R}^{max}_{o}-\mathcal{R}^* +\epsilon}{B}$
and $\mathbf{R}_{o}^{max}:= \sup_\mu \mathcal{R}^{\mathscr{M}^\times}_{o}({\mu})$ is the maximum possible objective reward for unconstrained $\mathscr{M}^\times$.
\end{lemma}

Lemma \ref{epsopt} suggests that if we can find an $\epsilon$-optimal mixed policy $\bar{\mu}$ of the sup-inf problem \eqref{bsupinf}, then the policy $\bar{\mu}$ is approximately optimal and approximately feasible in \eqref{prodprobform}, and therefore in Problem \eqref{probform} due to Theorem \ref{equiv1}. 

We use the exponentiated gradient (EG) algorithm \cite{hazan2016introduction} to find an $\epsilon$-approximate policy $\bar{\mu}$ for Problem \eqref{bsupinf}. This algorithm is used for online convex optimization and has an L1-norm bounded decision space. It is also known to satisfy the no-regret property, i.e., the average gap between its cumulative loss and that of the optimal hindsight decision asymptotically approaches zero. We refer the reader to \cite{hazan2016introduction} for more details of the algorithm.

Following the structure of the EG algorithm, our algorithm uses a sub-gradient of the following function of the dual variable:
\begin{equation}
    f(\boldsymbol{\lambda}) = \sup_\mu L(\mu,\boldsymbol{\lambda}). \label{eq:sup}
\end{equation}

The sub-gradient of function $f(\cdot)$ at $\boldsymbol{\lambda}$ is given by $\left[\mathcal{R}^{\mathscr{M}^\times}_{c}(\mu_{\boldsymbol{\lambda}})-\rho, \mathcal{R}^f(\mu_{\boldsymbol{\lambda}})-1+\delta \right]^{T}$,  where policy $\mu_{\boldsymbol{\lambda}}$ maximizes \eqref{eq:sup}. Such a policy exists  because the POMDP $\mathscr{M}^\times$ has finite state, observation, and action spaces~\cite{bertsekas1995dynamic}.
The \texttt{EG-CPOMDP} algorithm, which is described in detail in Algorithm \ref{algjr}, uses this sub-gradient to iteratively update $\boldsymbol{\lambda}$. The value of $\boldsymbol{\lambda}$ at the $k$th iteration is denoted by $\boldsymbol{\lambda}_k$ and the corresponding maximizing policy (that achieves $\sup_\mu L(\mu,\boldsymbol{\lambda}_k)$) is denoted by $\mu_k$.  In the $k$th iteration, computing the sub-gradient at $\boldsymbol{\lambda}_k$ involves two key steps: 
(a) $\textsc{opt}(\mathscr{M} ^{\times},\boldsymbol{\lambda}_k)$ which solves an unconstrained POMDP problem with expected reward given by  $L(\mu,\boldsymbol{\lambda}_k)$  and  returns the maximizing policy $\mu_k$; 
(b)   $\textsc{eval}(\mu_k, r^{c\times})$, which estimates $\mathcal{R}^{\mathscr{M}^\times}_{c}(\mu_k)$, and $\textsc{eval}(\mu_k, r^{f})$, which estimates $\mathcal{R}^f(\mu_k)$.

For a mixed policy $\mu_k = \sum_i \alpha_i \pi_i$, where $\pi_i$ are pure policies, $\textsc{eval}(\mu_k,r)$ is evaluated as $\sum_i \alpha_i \textsc{eval}(\pi_i,r)$. The algorithm does not depend on which methods are used for solving the unconstrained POMDP and evaluating the constraints as long as they satisfy the following assumption.

\begin{assumption}\label{exact}

The POMDP solver $\textsc{opt}(\mathscr{M} ^{\times},\boldsymbol{\lambda}_k)$ and the  expected reward estimators $\textsc{eval}(\mu_k, r^{c\times})$ and $\textsc{eval}(\mu_k, r^{f})$  used in Algorithm \ref{algjr} are exact, i.e.,
\begin{align*}
    & \textsc{opt}(\mathscr{M} ^{\times},\boldsymbol{\lambda}_k) = \arg\sup_\mu L(\mu,\boldsymbol{\lambda}_k) \quad \forall \ \boldsymbol{\lambda}_k, ~\text{and}~\forall \ \mu_k\\
    &\textsc{eval}(\mu_k,r^{c\times}) = \mathcal{R}^{\mathscr{M}^\times}_{c}(\mu_k)~\text{and}~\textsc{eval}(\mu_k,r^{f}) = \mathcal{R}^{f}(\mu_k).
\end{align*}
\end{assumption}
\begin{remark}
For solving an unconstrained POMDP, it is sufficient to consider pure policies and therefore, most solvers optimize only over the space of pure policies. Thus, the support size of $\mu_k$ in Algorithm \ref{algjr} is 1.
\end{remark}

Algorithm 1 runs for $K$ iterations and  returns a policy  
\(\bar{\mu}= {(1/K)}{\sum_{k=1}^K\mu_k}\), which  is a mixed policy that assigns a probability of $1/K$ to each $\mu_k$ for $k=1,\cdots, K$. The following theorem states that the  policy $\bar{\mu}$ obtained from Algorithm \ref{algjr}  is an approximately optimal and feasible policy for Problem \eqref{prodprobform}.

\begin{algorithm}[tb]
  \caption{\texttt{EG-CPOMDP} Algorithm}
  \label{algjr}
\begin{algorithmic}
    \STATE Input: Constrained product POMDP $\mathscr{M} ^{\times}$, learning rate $\eta$
    \STATE Initialize $\boldsymbol{\lambda}_1 = \left[ B/3, B/3 \right]$
  \FOR{$k=1,\dots,K$}
\STATE $\mu_k \leftarrow \textsc{opt}(\mathscr{M} ^{\times},\boldsymbol{\lambda}_k) $
\STATE $\hat{p}_k \leftarrow \textsc{eval}(\mu_k,r^{f}) $
\STATE $\hat{r}^{c\times}_k \leftarrow \textsc{eval}(\mu_k,r^{c\times}) $
\STATE
\STATE $\frac{\partial f}{\partial \lambda^{f}}\leftarrow \hat{p}_k-1+\delta$
\STATE $\frac{\partial f}{\partial \lambda^{c}}\leftarrow \hat{r}^{c\times}_k-\rho$
\STATE
\STATE $\lambda^{f}_{k+1} \leftarrow B\frac{\lambda^{f}_k e^{-\eta(\hat{p}_k-1+\delta)}}{B+\lambda^{f}_k(e^{-\eta(\hat{p}_k-1+\delta)}-1) + \lambda^{c}_k(e^{-\eta(\hat{r}^{c\times}_k-\rho)}-1)}$
\STATE $\lambda^{c}_{k+1} \leftarrow B\frac{\lambda^{c}_k e^{-\eta(\hat{r}^{c\times}_k-\rho)}}{B+\lambda^{f}_k(e^{-\eta(\hat{p}_k-1+\delta)}-1) + \lambda^{c}_k(e^{-\eta(\hat{r}^{c\times}_k-\rho)}-1)}$
  \ENDFOR
\STATE Output: $\bar{\mu}=\frac{\sum_{k=1}^K\mu_k}{K}$, $\bar{\boldsymbol{\lambda}}=\frac{\sum_{k=1}^K\boldsymbol{\lambda}_k}{K}$
\end{algorithmic}
\end{algorithm}

\begin{theorem}\label{lagrangethm}
Under Assumptions \ref{finiteassumpstrong} and \ref{exact}, if $\eta = \sqrt{\frac{\log 3}{2KB^2G^2}}$ in Algorithm \ref{algjr}, the policy $\bar{\mu}$ returned by Algorithm \ref{algjr} satisfies:
\begin{align}
    \mathcal{R}^{\mathscr{M}}_{o}(\bar{\mu}) &\geq \mathcal{R}^*-2BG\sqrt{2\log3/K}, \\
     \mathcal{R}^{\mathscr{M}}_{c}(\bar{\mu}) &\geq \rho - \left( \frac{\mathbf{R}^{max}_{o}  - \mathcal{R}^* + 2BG\sqrt{2\log3/K}}{B} \right),\notag \\
     \Py_{\varphi}^{\mathscr{M}}(\bar{\mu})&\geq 1-\delta  - \left( \frac{\mathbf{R}^{max}_{o}  - \mathcal{R}^* + 2BG\sqrt{2\log3/K}}{B} \right). \notag
\end{align}
where $G = \max\{\mathbf{R}^{max}_{c},1\}$.
\end{theorem}
 If \eqref{probform} is feasible, then $\mathcal{R}^*$ is finite and Theorem \ref{lagrangethm} guarantees that Algorithm \ref{algjr} returns an approximately optimal and approximately feasible policy $\bar{\mu}$. If \eqref{probform} is not feasible, then $\mathcal{R}^* = - \infty$, hence the 3 inequalities in Theorem \ref{lagrangethm} are trivially true. In either case, we can  determine from Algorithm 1 how far $\bar{\mu}$ is from being feasible. This is because, in the $k$th iteration, the algorithm  evaluates the constraints  for policy $\mu_k$. The average of these constraint values is exactly the constraint value for policy  $\bar{\mu}$.

In Theorem \ref{lagrangethm}, we make use 
 of the assumption that Algorithm \ref{algjr} has access to an exact unconstrained POMDP solver and an exact method for evaluating $\mathcal{R}^{\mathscr{M}^\times}_{c}(\mu)$ and $\mathcal{R}^f(\mu)$ (see  Assumption \ref{exact}). In practice, however, methods for solving POMDPs and evaluating policies are approximate. We next describe the performance of Algorithm \ref{algjr} under the more practical  assumption that the POMDP solver returns an $\epsilon$-optimal policy and the evaluation function has an $\epsilon$-error.

\begin{assumption}\label{estimate}
The POMDP solver $\textsc{opt}(\mathscr{M} ^{\times},\boldsymbol{\lambda}_k)$ and the 
expected reward estimators $\textsc{eval}(\mu_k, r^{f})$ and $\textsc{eval}(\mu_k, r^{c\times})$ used in Algorithm \ref{algjr} are approximate, i.e., for all $\boldsymbol{\lambda}_k$ and $\mu_k = \textsc{opt}(\mathscr{M} ^{\times},\boldsymbol{\lambda}_k)$, we have 
\begin{align*}
& \sup_\mu L(\mu,\boldsymbol{\lambda}_k) - L(\mu_k,\boldsymbol{\lambda}_k) \leq \epsilon_{bp};  \\
         &|\textsc{eval}(\mu,r^{c\times}) - \mathcal{R}^{\mathscr{M}^\times}_{c}(\mu)| \leq \frac{\epsilon_{est}}{2B} \quad \forall \ \mu,  \quad  ~~\mbox{and}\\
         &|\textsc{eval}(\mu,r^{f}) - \mathcal{R}^{f}(\mu)| \leq \frac{\epsilon_{est}}{2B}.
\end{align*}
\end{assumption}

A result similar to Theorem \ref{lagrangethm} can be obtained under Assumption \ref{estimate} by using similar arguments.

\begin{theorem}\label{lagrangeapproxthm}
Under Assumptions \ref{finiteassumpstrong} and \ref{estimate}, if $\eta = \sqrt{\frac{\log 3}{2KB^2G^2}}$ in Algorithm \ref{algjr}, the policy $\bar{\mu}$ returned by Algorithm \ref{algjr} satisfies:
\begin{align}
    \mathcal{R}^{\mathscr{M}}_{o}(\bar{\mu}) &\geq \mathcal{R}^*-\left(2BG\sqrt{2\log3/K}+\epsilon_{tot}\right), \\
     \mathcal{R}^{\mathscr{M}}_{c}(\bar{\mu}) &\geq \rho - \left(\frac{\mathbf{R}^{max}_{o} - \mathcal{R}^* + 2BG\sqrt{2\log3/K}+\epsilon_{tot}}{B}\right), \nonumber\\
     \Py_{\varphi}^{\mathscr{M}}(\bar{\mu})&\geq 1-\delta -  \left(\frac{\mathbf{R}^{max}_{o} - \mathcal{R}^* + 2BG\sqrt{2\log3/K}+\epsilon_{tot}}{B}\right).\nonumber
\end{align}
where $\epsilon_{tot} = \epsilon_{bp} + 2\epsilon_{est} $ and $G = \max\{\mathbf{R}^{max}_{c},1\}$. 
\end{theorem}

In Lemma~\ref{epsopt} and Theorems \ref{lagrangethm} and \ref{lagrangeapproxthm}, we can replace $\mathcal{R}^*$ by its lower bound (e.g., replacing $\mathcal{R}^*$ by $0$ in scenarios with non-negative objective rewards) to obtain slightly relaxed, yet more computable lower bounds for the performance of the returned policy.

\begin{remark}
Algorithm \ref{algjr} uses $K$ pure policies to generate the mixed policy $\bar{\mu}$. If we wish to obtain a mixed policy with a small support, we can define $\bar{\mu} = \sum_{k=1}^K w_k \mu_k$, where $w_k$s are  a \emph{basic feasible solution} (BFS) of the following linear program. 
\begin{equation} \tag{BFS}\label{bfs}
    \begin{aligned}
  \underset{w \geq 0, \lVert w \rVert_{1} \leq 1}{\sup } \quad & \sum_{k=1}^K w_k{\mathcal R}^{\mathscr{M}^\times}_{o}(\mu_k)\\ 
    \mathrm{s.t.} \quad  &\sum_{k=1}^K w_k{\mathcal R}^{\mathscr{M}^\times}_{c} \geq  \rho-o(1/\sqrt{K}),\\
& \sum_{k=1}^K w_k\mathcal{R}^{f}(\mu_k) \geq  1-\delta-o(1/\sqrt{K}).
\end{aligned}
\end{equation}
This leads to a mixed policy $\bar{\mu}$ whose support size is at most three.
\end{remark}

We next consider three special cases of our problem with different types of time horizon $T$.
\begin{remark}
When the horizon $T$ is a constant and the constraint reward $r^{c}_{t}$ is non-negative, Assumption \ref{finiteassumpstrong} is trivially true and, therefore, Theorem \ref{lagrangethm} holds. 
\end{remark}

\begin{remark}
Let $\{E_t: t=0,1,2,...\}$ be a sequence of i.i.d. Bernoulli random variables with $\Py[E_0 = 1] = 1-\gamma$, $\gamma<1$. Let the geometrically-distributed time-horizon $T$ be defined as
\begin{align}
    T = \min\{t: E_t = 1\}.
\end{align}
The random horizon $T$ has an expected value of   $\gamma/(1-\gamma)$ under any policy and thus satisfies Assumption \ref{finiteassumpstrong} if the constraint reward $r^{c}_{t}$ is non-negative. The following lemma shows that with such a geometric time horizon, the policy optimization problem for $L(\mu, \boldsymbol{\lambda}_k)$ is equivalent to solving an infinite horizon  discounted-reward POMDP.

\begin{lemma}\label{disc}
For a given $\boldsymbol{\lambda}$, maximizing $L(\mu,\boldsymbol{\lambda})$ over $\mu$ under a geometrically distributed time horizon is equivalent to maximizing the following infinite horizon discounted reward

{
\small 
\begin{align*}
    \E_\mu\left[\sum_{t=0}^\infty \gamma^{t}\left(r_t^{o\times}(X_t,\U_t) + \frac{\lambda^{f}(1-\gamma)}{\gamma}r^f(X_{t}) +  \lambda^{c}r_t^{c\times}(X_t,\U_t) \right) \vphantom{\sum_{t=0}^\infty} \right].
\end{align*}
}
\end{lemma}
\end{remark}
Thus, by using appropriate POMDP solvers and a Monte Carlo method for constraint reward evaluation, we can implement Algorithm \ref{algjr} for constant and geometrically distributed time horizons.

\paragraph*{Goal-POMDPs}\label{goalpomdp}
In Goal-POMDPs, there is a set of goal states $GO$ and the POMDP run terminates once the goal state is reached. We assume that every possible action in every non-goal state results in a  strictly positive cost (or strictly negative reward). The objective is to obtain a policy that maximizes the total expected reward while the probability of satisfying the specification $\varphi$  is at least $1-\delta$. (In this section, for simplicity of discussion, we do not consider a constraint reward function.) Under this reward assumption, for any given policy, the expected time taken to reach the goal is infinite if and only if the total expected reward is negative infinity. Assumption \ref{finiteassumpstrong} may not be true a priori in this setting. However, we can exclude every policy with negative infinite expected reward without any loss of optimality. After this exclusion of policies, Assumption \ref{finiteassumpstrong} holds and all our results in the previous sections are applicable.

Consider the product POMDP $\mathscr{M} ^{\times} = (\mathcal{S}^{\times},  \mathcal{U}^{\times},P^{\times},s_{0}^{\times},r^{o\times},r^f,\varpi,{O},Z^\times)$. Let $GO \subseteq \mathcal{S}^\times$ be a set of goal states. For every non-goal state-action pair $(x,u)$ in the product MDP, let $r^{o\times}(x,u) <0$. This is a constrained Goal-POMDP, which can be solved using the Lagrangian approach discussed earlier. 
The problem of optimizing the Lagrangian function for a given $\lambda$ can be reformulated as an unconstrained Goal-POMDP with minor modifications which can be solved using solvers like Goal-HSVI \cite{horak2018goal}.
We make modifications to ensure that there is a single goal state and the rewards for every non-goal state action pair are strictly negative. The rewards and transitions until time $T$ are the same as in the product MDP. We replace the goal states in $GO$ with a unique goal state $g$. When a state $x \in GO$ is reached (at $T$), the process goes on for two more time steps $T+1$ and $T+2$. At time $T+1$, the agent receives a reward
\begin{align}
    \lambda(r^f(X_{T+1})-2+\delta).
\end{align}
This ensures that the reward is strictly negative. At $T+2$, the agent reaches the goal state $g$.  One can easily show that this modified Goal-POMDP is equivalent to optimizing the Lagrangian function for a given $\lambda$.

\begin{remark}
The framework developed in this section can easily be generalized to the setting where there are multiple $\ltlf$ specifications and reward constraints.  
\end{remark}

\section{Multi-Agent Systems}
\label{sec:ma-systems}

In previous sections, we have discussed how we can incorporate $\ltlf$ specifications in a single-agent POMDP and solve it using a Lagrangian approach. We will now consider a setting in which there are multiple agents and these agents select their actions based on different information. In this section, for the sake of simplicity in presentation, we will consider a single $\ltlf$ specification, and no reward constraints.

We can solve a multi-agent problem with $\ltlf$ constraints by first converting it into a problem with reward-like constraints as we did in Section \ref{conspomdp}. We can then use the exponentiated gradient algorithm to solve this constrained multi-agent problem. In the exponentiated gradient algorithm, we need to solve unconstrained multi-agent problems. One approach for solving a large class of unconstrained multi-agent problems is the \emph{common information approach} \cite{nayyar2013decentralized}. This approach transforms the unconstrained multi-agent problem into a single-agent unconstrained POMDP (with enlarged state and action spaces) which, in principle, can be solved using POMDP solvers. While this approach is conceptually sound, it is computationally intractable in general. We now discuss a multi-agent model with a specific structure and illustrate how some of the computational issues can be mitigated.

We consider a class of labeled multi-agent systems which can be characterized by a tuple $ \mathscr{M} = (N,{\mathcal{S}},\mathcal{S}^{\mathrm{loc}},{\U},P,\loc{P},\varpi,A P,L,r,T)$. $N$ is the number of agents in the system, ${\mathcal{S}}$ is a finite \emph{shared} state space, and ${\U}$ is a finite \emph{joint} action space. The shared state at time $t$ is denoted by $\mathcal{S}_t$; the action of agent $i$ at time $t$ is denoted by $\U_t^i$; and $\U_t = (\U_t^1, \dots, \U_t^N)$ denotes the joint action. 
$P_t(s,u;s')$ is the probability of transitioning from state $s$ to state $s'$ on taking joint action $u$. The shared state and the agents' actions can be observed by all the agents in the system. Each agent $i$ has an associated local state $\loci{\mathcal{S}}{i}_t$ at time $t$. The evolution of agent $i$'s local state is captured by the local transition function $\loci{P}{i}_t:\loci{\mathcal{S}}{i}\times{\mathcal{S}} \times {\U} \to \Delta _{\loci{\mathcal{S}}{i}}$ where $\loci{P}{i}_t(\loci{s}{i},s,u;\loci{s}{i'})$ is the probability of transitioning to  local state $\loci{s}{i'}$ if the current local state is $\loci{s}{i}$, the current global state is $s$, and the joint action $u$ is taken. For convenience, let us denote $\bar{\mathcal{S}} = \mathcal{S} \times \loc{\mathcal{S}}$, where $\loc{\mathcal{S}} = \prod_{i=1}^N\loci{\mathcal{S}}{i} $. Note that the local state transitions among agents are mutually independent given the shared state and action, and the state dynamics can be viewed as
\begin{align}
    \mathcal{S}_{t+1} &= \zeta^g_t(\mathcal{S}_t,\U_t,W_t^g)\\
    \loci{\mathcal{S}}{i}_{t+1} &= \zeta^g_t(\loci{\mathcal{S}}{i}_t,\mathcal{S}_t,\U_t,W_t^i),
\end{align}
where $\zeta^g_t$ and $\zeta^i_t$ are fixed transformations and $W_t^g, W_t^i$ are the noise variables associated with the dynamics.
$\varpi \in \Delta_{\bar{\mathcal{S}}} $ is the initial state distribution. $AP$ is a set of atomic propositions, $L: \bar{\mathcal{S}} \to 2^{AP}$ is a labeling function which indicates the set of atomic propositions which are true in each state, 
$r_t: \bar{\mathcal{S}} \times \U \to \mathbb{R}$ is a reward function.
The system runs for a random time horizon $T$. This random time may be independent of the multi-agent system, or could be a stopping time with respect to the common information.

\subsubsection*{Information Structure} 

Agents have access to different information. The information that agent $i$ can possibly use to select its actions at time $t$ is given by
\begin{align}
I_t^i = \{\mathcal{S}_{0:t},\U_{0:t-1},\loci{\mathcal{S}}{i}_{0:t}\}.
\end{align}
The \emph{common information} available to all the agents at time $t$ is denoted by
\begin{align}
C_t = \{\mathcal{S}_{0:t},\U_{0:t-1}\},
\end{align}
and the \emph{private information} that is available only to agent $i$ at time $t$ is denoted by
\begin{align}
P_t^i = \{\loci{\mathcal{S}}{i}_{0:t}\}.
\end{align}

This information structure is referred to as the control-sharing information structure \cite{controlsharing}.
A \emph{control law} $\pi_t^i$ for agent $i$ maps its information $I_t^i$ to its action $\U^i_t$, i.e., ${\U}_t^i = \pi^i_t(I_t^i)$. The collection of control laws $\pi^i := (\pi_{0}^i,\pi_1^i,\dots,)$ over the entire horizon is referred to as agent $i$'s \emph{policy}. The set of all such policies for agent $i$ is denoted by $\mathcal{P}^i$. The team's policy is denoted by $\pi = (\pi^1, \dots,\pi^N)$. We refer to such deterministic team policies as pure policies and denote the set of all pure team policies with $\mathcal{P}$.

The set of mixed policies is defined in the same manner as in \eqref{mixedpol}. Given a mixed policy $\mu$, the team randomly selects a pure policy $\pi$ with probability $\mu(\pi)$. This random selection happens using shared randomness which can be achieved by the agents in the team through a pseudo-random number generator with a common seed.

A \emph{run} ${\xi}$ of the POMDP is the sequence of  states and actions $(\bar{\mathcal{S}}_0, \U_0)(\bar{\mathcal{S}}_1,\U_1)\cdots (\bar{\mathcal{S}}_{T},\U_{T})$. The total expected reward associated with a team policy $\mu$ is given by
\begin{align}
    \mathcal{R}^{\mathscr{M}}(\mu) &= \E_\mu^\mathscr{M}\left[\sum_{t=0}^T r_t(\bar{\mathcal{S}}_t,\U_t)\right]\\
    &=\sum_{\pi \in \supp(\mu)}\left[\mu(\pi)\E_\pi^\mathscr{M}\left[\sum_{t=0}^T r_t(\bar{\mathcal{S}}_t,\U_t)\right]\right].
\end{align}

Let $\varphi$ be an $\ltlf$ specification defined using the labeling function $L$. The probability that a run of $\mathscr{M}$ satisfies $\varphi$ under policy $\mu$ is denoted by $\Py_{\varphi}^{\mathscr{M}}(\mu)$. We want to solve the following constrained optimization problem for the team of agents:

\begin{equation}\tag{MP1} \label{multiprobform}
    \begin{aligned}
   \textbf{LTL$_f$-MA:}~~~\underset{\mu}{\sup} \quad & {\mathcal R}^{\mathscr{M}}(\mu)\\ \mathrm{s.t.} \quad  & \Py_{\varphi}^{\mathscr{M}}(\mu) \geq  1-\delta.
\end{aligned}
\end{equation}

\subsection{Constrained Product Multi-Agent Problem}\label{genconsma}

We construct a constrained product multi-agent problem similar to the constrained product POMDP in Section \ref{conspomdp}. In this construction, the system state space is $\bar{\mathcal{S}}$. Let the DFA associated with the specification $\varphi$ be  $\mathscr{A} = (Q, \Sigma, q_0, \delta, F)$. Hence, the product state space is $\bar{\mathcal{S}} \times Q$. The action space is $A$. The transitions and rewards are constructed in the same manner discussed in Section \ref{conspomdp}. This  leads to the following multi-agent problem with reward-like constraints:

\begin{equation} \tag{MP2}\label{multiprodprobform}
    \begin{aligned}
   \textbf{C-MA:}~~~\underset{\mu}{\sup} \quad & {\mathcal R}^{\mathscr{M}^\times}(\mu)\\ \mathrm{s.t.} \quad  & \mathcal{R}^{f}(\mu) \geq  1-\delta.
\end{aligned}
\end{equation}

\subsection{Global and Local Specifications}
Consider that the specification $\varphi$ admits additional structure.
For each agent, $L^i:\mathcal{S}\times \loci{\mathcal{S}}{i}\to 2^{AP}$  is a labeling function  and $\varphi^i$ a local specification  defined with respect to it. There is a shared labeling function $L^g: \mathcal{S} \to 2^{AP}$ and a shared specification $\varphi^g$ defined with respect to $L^g$. The overall specification $\varphi$ is the conjunction of the local and shared specifications, i.e.,
\begin{align}
    \varphi = \varphi^g \wedge \varphi^1 \wedge \dots \wedge \varphi^N.
\end{align}

For each local specification, let the corresponding DFA be $\mathscr{A}^i = (Q^i, \Sigma^i, q^i_0, \delta^i, F^i)$ and for the shared specification, let the corresponding DFA be $\mathscr{A}^g = (Q^g, \Sigma^g, q^g_0, \delta^g, F^g)$.

Instead of using the product construction in the previous section, we can construct the product state as
\begin{align*}
    (\bar{\mathcal{S}}_t, Q^g_t, Q_t^1,\dots,Q_t^N) &= (\mathcal{S}_t,Q_t, \loci{\mathcal{S}}{1}_t,Q_t^1,\dots,\loci{\mathcal{S}}{N}_t,Q_t^N)\\
    &=: (X_t^g,X_t^1, \dots, X_t^N).
\end{align*}
The shared product state $X_t^g$ evolves as
\begin{align}
    X_{t+1}^g &= (\mathcal{S}_{t+1},Q_{t+1})\notag\\
    &= (\zeta_t^g(\mathcal{S}_t,\U_t,W^g_t),\delta^g(Q_t,L^g(\mathcal{S}_t))).
\end{align} 
The local product state $X_t^i$ evolves as
\begin{align}
    X_{t+1}^i &= (\loci{\mathcal{S}}{i}_{t+1},Q_{t+1}^i) \notag\\
    &= (\zeta_t^i(\mathcal{S}_t,\loci{\mathcal{S}}{i}_t,\U_t,W^i_t),\delta^i(Q_t^i,L^i(\mathcal{S}_t,\loci{\mathcal{S}}{i}_t))).
\end{align}
The reward is 
\begin{align}
    r^{\times}_{t}((\bar{s},q^g,q^1,\dots,q^N),u) &= {r}_t(\bar{s},u),~ \forall \bar{s},q,u\label{prodrewdefm}\\
    \label{finalrewdef}r^f((\bar{s},q^g,q^1,\dots,q^N)) &=
    \begin{cases}
    1, &\text{if }q^g \in F^g, q^i\in F^i~ \forall i\\
    0, &\text{otherwise}. 
    \end{cases} 
    \notag
\end{align}
Under this specification structure and modified construction, the \emph{product problem} also conforms to the control sharing model \cite{controlsharing}.
For this model, it has been shown that there exist optimal policies in which agent $i$ uses \emph{reduced} private information $\{\loci{\mathcal{S}}{i}_{t},Q_t^i\}$ (as opposed to $\{\loci{\mathcal{S}}{i}_{0:t},Q_{0:t}^i\}$) and the common information $C_t = \{\mathcal{S}_{0:t},Q_{0:t},\U_{0:t-1}\}$ to choose its actions. This can be formally stated as the lemma below.

\begin{lemma}[Policy Space Reduction]\label{policy_reduction}
In Problem \eqref{multiprodprobform}, we can restrict our attention to control laws of the form below without loss of optimality
    \begin{align}
        \U_t^i = \pi_t^i(\loci{\mathcal{S}}{i}_{t},Q_t^i,\mathcal{S}_{0:t},\U_{0:t-1}).
    \end{align}
\end{lemma}
\begin{proof}
    Given an arbitrary team policy $\pi$, we can construct a team policy $\bar{\pi}$ of the form above that achieves the same expected reward (including the constraint). See~\cite[Proposition 3]{controlsharing} for a complete proof.
\end{proof}

\begin{remark}
The substantial reduction in policy space under Lemma \ref{policy_reduction} is achieved by the modified construction of the product multi-agent problem. Using the construction in Section \ref{genconsma} would not have enabled us in achieving this reduction.
    A similar private information reduction can be achieved for models with transition independence \cite{kartik2022optimal}.
\end{remark}

\subsection{Algorithm}

    We can use Algorithm \ref{algjr} to solve Problem \eqref{multiprodprobform}.
    The unconstrained solver \textsc{opt} in Algorithm \ref{algjr} is now an unconstrained multi-agent solver. One way to implement such a solver is to use the common information approach \cite{nayyar2013decentralized}. Under assumptions similar to the ones made in the previous sections, owing to similar reward structures in the constrained multi-agent problem, the results derived in the previous sections can be extended to this setting in a straightforward manner.

\section{Experiments}\label{exp}

We consider a collection of gridworld problems in which an agent or a team of agents needs to maximize their reward while satisfying an $\ltlf$ specification. We consider three types of tasks, i.e., reach-avoid, ordered, and reactive tasks. In all of our experiments, we use the SARSOP \cite{kurniawati2008sarsop} solver for finding the optimal policy $\mu_k$ and a default discount factor of $0.99$. 
To estimate the constraint function, we use Monte Carlo simulations. 
We  use the online tool LTL\textsubscript{f}2DFA \cite{francesco_fuggitti_2019_3888410} based on MONA \cite{monamanual2001} to generate an equivalent DFA for an $\ltlf$ formula.

For each location $(i,j)$ in the grid, $L[(i,j)]$ is the label assigned to it by the labeling function $L$ (see Section~\ref{sec:prelim_a}). The images corresponding to the various models indicate the grid space and the associated labeling function, e.g., in Fig.~\ref{fig:model2}, we have $L[(5,2)] = \{b\}, L[(1,6)] = \{b\}, L[(7,7)] = \{a\}$ and $L[(i,j)] = \{\}$ for all other grid locations $(i,j)$. In all single agent models, the agent starts from the grid location $(0,0)$. Further, the default reward for all actions is $0$ in all grid locations, unless specified otherwise. 

In the experiments of Section~\ref{exp:1} and~\ref{exp:2}, the agent's transitions in the gridworld are stochastic, i.e., if the agent decides to move in a certain direction, it moves in that direction with probability $0.95$ and,  with probability $0.05$, it randomly moves one step  in any direction that is not opposite to its intended direction. Further, the agent's observation of its current location is   noisy -- it is equally likely to be the agent's true current location or any location neighboring its current location.

For each model discussed below, we use Algorithm \ref{algjr} to generate a mixed policy $\bar{\mu}$. The corresponding reward $\mathcal{R}^\mathscr{M}(\bar{\mu})$ and constraint $\mathcal{R}^f(\bar{\mu})$ are shown in Table \ref{tab:data}. We observe that the probability of satisfying the constraint generally exceeds the required threshold. Occasionally, the constraint is violated albeit only by a small margin. This is consistent with our result in Theorem \ref{lagrangethm}.
We also observe that the agent behaves in a manner that achieves high reward in all of these models. 

\subsection{Single-Agent Location Uncertainty with Reach-Avoid Task}\label{exp:1}

In this task, we are interested in reaching a goal state $a$ and always avoiding unsafe states $b$. This can be specified using $\ltlf$ as $ \varphi_1 = \textbf{F}a\wedge(\textbf{G}\neg b)$. This task was performed on model $\mathscr{M}_1$ with discounted reward $r((1,6),u) = 3, r((4,3),u) = 3$, and $r((7,7),u) = 1$ for all actions $u$. 

In this experiment, we observe two characteristic behaviors. The agent reaches the goal state $a$ and remains there. This behavior ensures that the specification is met but the reward is relatively lower. Following the other behavior, the agent goes towards the location $(4,3)$ and tries to remain there to obtain higher reward. However, since the obstacle is very close and the transitions are stochastic, it is prone to violating the constraint. Nonetheless, this violation is rare enough and the overall satisfaction probability exceeds the desired threshold.

\begin{figure}[h]
    \centering
          \begin{subfigure}[b]{0.17\textwidth}
         \centering
         \includegraphics[width=\textwidth]{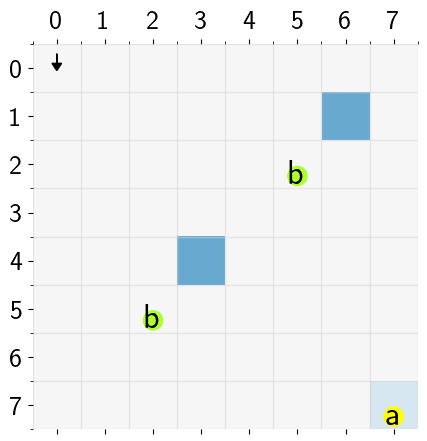}
         \caption{Model $\mathscr{M}_1$.}
         \label{fig:model2}
     \end{subfigure}
     \hfill
         \begin{subfigure}[b]{0.17\textwidth}
         \centering
         \includegraphics[width=\textwidth]{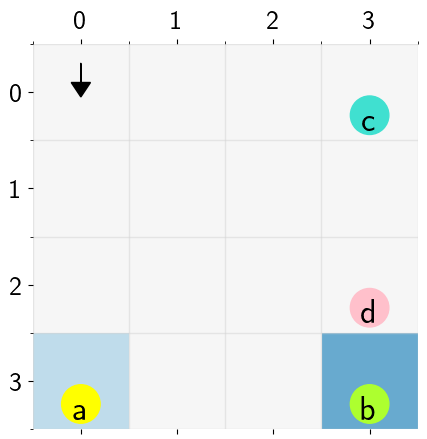}
         \caption{Model $\mathscr{M}_2$.}
         \label{fig:model6}
     \end{subfigure}
\end{figure}

\subsection{Single-Agent Location Uncertainty with Reactive Task}\label{exp:2}

In this task, there are four states of interest: $a,b,c$, and $d$. The agent must eventually reach $a$ or $b$. However, if it reaches $b$, then it must visit $c$ without visiting $d$. This can be expressed as $\varphi_2= \textbf{F}(a \vee b) \wedge \textbf{G}(b \to (\neg {d}  \textbf{U}c))$. This task was performed on 
model $\mathscr{M}_2$ with discounted reward $r((3,0),u) = 1$ and $r((3,3),u) = 2$ for all actions $u$.

In this experiment, under the policy returned by Algorithm~1 for $B=10$ and $1-\delta = 0.8$,  the agent goes to $a$ and remains there, thus satisfying the constraint. Occasionally, the agent also goes to state $b$ and remains there to obtain a larger reward. This behavior violates the constraint, since if the agent ever visits $b$, it must eventually go to $c$. We observe that this occurrence is rare enough and the overall satisfaction probability is close to the desired threshold. 

We extend our experiment by varying $B$ and $\delta$, presenting the results in Table \ref{tab:sensitivity}. All the final policies either exceed or narrowly miss the desired satisfaction probability threshold $1-\delta$.
Notably, as $B$ increases while keeping $\delta$ constant, the satisfaction probability tends to rise with minimal impact on the total objective value. This observations aligns with the bounds presented in Theorem \ref{lagrangeapproxthm}.

In the experiments of Section~\ref{exp:3}, \ref{exp:4}, and \ref{exp:5}, the agents' transitions in the gridworld are deterministic. 

\begin{figure}[h]
     \centering
     \begin{subfigure}[b]{0.17\textwidth}
         \centering
         \includegraphics[width=\textwidth]{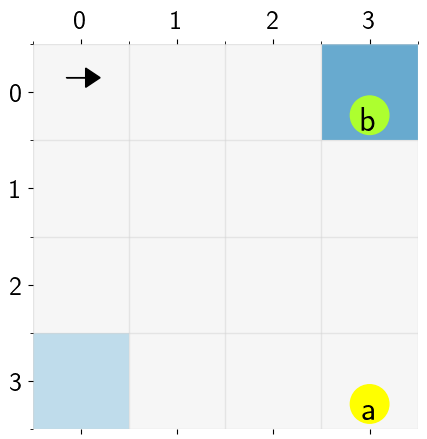}
         \caption{Obstacle at $(0,3)$.}
         \label{fig:model8a}
     \end{subfigure}
     \hfill
     \begin{subfigure}[b]{0.17\textwidth}
         \centering
         \includegraphics[width=\textwidth]{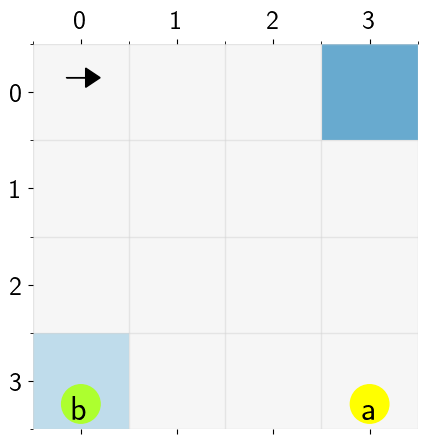}
         \caption{Obstacle at $(3,0)$.}
         \label{fig:model8b}
     \end{subfigure}
     \caption{Model $\mathscr{M}_3$ Reach-Avoid Task.}
    \label{fig:model4}
\end{figure}

\begin{table}
    \centering
    \caption{Results for different $B$ and $\delta$ for the setting of subsection VII-B}\label{tab:sensitivity}
    \begin{tabular}{lccccccccccccc}
      \toprule 
      & $\mathcal{R}^{\mathscr{M}}(\bar{\mu})$ & $\mathcal{R}^f(\bar{\mu})$ & $1-\delta$ & $B$ & $\eta$ & $K$  \\
      \midrule 
     &$1.33$ & $0.49$ & $0.5$ & $5$ & $2$ & $100$ \\ 
       & $1.31$ & $0.51$ & $0.5$ & $50$ & $2$ & $100$ \\ 
       & $1.31$ & $0.52$ & $0.5$ & $500$ & $2$ & $100$  \\ 
       & $1.14$ & $0.69$ & $0.7$ & $5$ & $2$ & $100$ \\ 
       & $1.12$ & $0.71$ & $0.7$ & $50$ & $2$ & $100$ \\ 
       & $1.11$ & $0.72$ & $0.7$ & $500$ & $2$ & $100$  \\
       & $0.95$ & $0.89$ & $0.9$ & $5$ & $2$ & $100$ \\
       & $0.94$ & $0.90$ & $0.9$ & $50$ & $2$ & $100$\\
        & $0.92$ & $0.92$ & $0.9$ & $500$ & $2$ & $100$ \\
      \bottomrule 
    \end{tabular}
\end{table}

\subsection{Single-Agent Predicate Uncertainty with Reach-Avoid Task}\label{exp:3}

In this task, the reach-avoid specification  is the same as $\varphi_1$ in Section \ref{exp:1} and was performed on model $\mathscr{M}_3$ with discounted reward $r((3,0),u) = 2$ and $r((0,3),u) = 4$ for all actions $u$. However, the agent does not know which location to avoid, as there is uncertainty in the location of the object that the agent has to avoid. There are two possible locations for object $b$: $(3,0)$ and $(0,3)$. In both cases, whenever the agent is far away (Manhattan distance greater than 1) from the object $b$, it gets an observation `F' indicating that it is \emph{far} with probability $1$. When the object is at the bottom left and the agent is adjacent to it, the agent gets an observation `C' with probability $0.9$ indicating that the object is \emph{close}. However, if object $b$ is at the top right and the agent is adjacent to it, the agent gets an observation `C' only with probability $0.1$. Therefore, the detection capability of the agent is stronger when the object is in the bottom-left location as opposed to when it is in the top-right location.

In this experiment, the agent receives high reward when it remains in the top-right corner compared to a moderate reward in the bottom-left corner. Further, the agent's detection capability is better when it is in the bottom-left region than when it is in the top-right region. Thus, it generally first heads towards the location $a$ (since it has to eventually visit it) via the bottom-left region without hitting the corner and acquires information on where the object is located. After reaching $a$, it goes to the top-right corner if the obstacle is \emph{not} located there, and bottom-left corner otherwise. We see rare instances where the agent completely ignores the constraint and just maximizes the reward.

\begin{figure}
    \centering
    \includegraphics[scale=0.42]{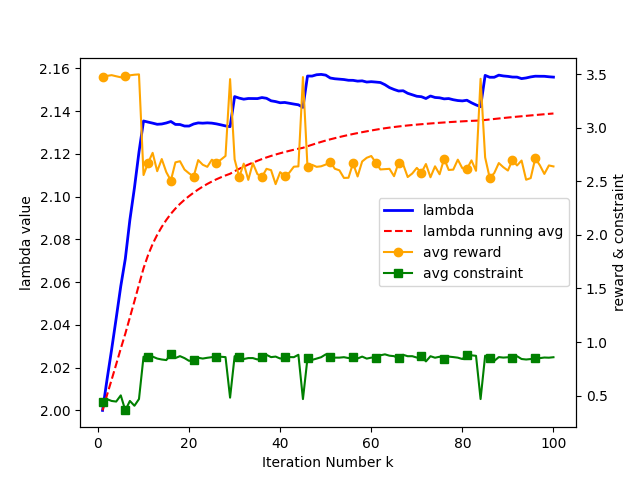}
    \caption{This plot illustrates how the Lagrange multiplier $\lambda_k$, the reward $\mathcal{R}^\mathscr{M}(\mu_k)$, and the probability of satisfaction $\Py_{\varphi}^{\mathscr{M}}(\mu_k)$ evolve with $k$ for the experiment in Section~\ref{exp:3}.}
    \label{fig:my_label}
\end{figure}

Figure \ref{fig:my_label}  illustrates the performance of the various policies $\mu_k$. The Lagrange multiplier $\lambda_k$ shows a decreasing trend as long as the constraint continues to be met (which happens for the majority of the iterations). However, the multiplier eventually diminishes to a point where the constraint is breached and 
we observe noticeable increase in reward. These spikes contribute to the overall average reward. Given the significant  
constraint violation, the Lagrange multiplier experiences an uptick. This cycle ensures that constraint violation occurs rarely. Since we pick a policy randomly and with uniform distribution, the average error probability is still close to the threshold (see Table \ref{tab:data}).

\subsection{Goal-POMDP Specification Uncertainty with Random Ordered Task}\label{exp:4}

In this task, the agent needs to visit state $a$ and $b$ strictly in that order or \emph{vice versa}. Thus, the uncertainty lies in the specification which the agent is required to satisfy. Either specification is chosen with uniform probability at the start (indicated by the truth value of $c$) and kept fixed. 
The information regarding the truth value of $c$ is unknown to the agent until it is revealed by an observation from a particular gridworld state.  This setting can be specified using $\ltlf$ as $ \varphi_4 = (c \to
(\neg b \textbf{U} (a \wedge \textbf{F}b))) \wedge (\neg c \to (\neg a \textbf{U} (b \wedge \textbf{F}a)))$.

\begin{figure}[h]
     \centering
     \begin{subfigure}[b]{0.17\textwidth}
         \centering
         \includegraphics[width=\textwidth]{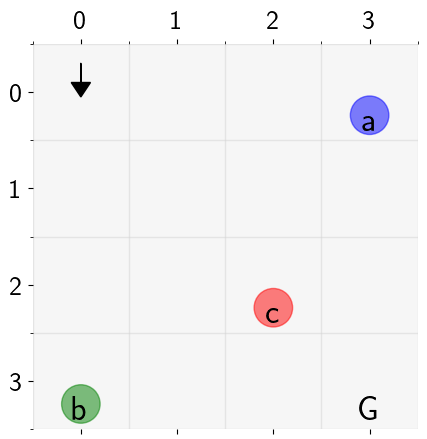}
     \end{subfigure}
     \caption{Model $\mathscr{M}_4$ Random Ordered Task with Goal.}
         \label{fig:model13}
\end{figure}
This task was performed on model $\mathscr{M}_4$ with the observation corresponding to grid location $(2,2)$ providing the truth value of $c$. In addition, the agent seeks to further reach the goal state, i.e., grid location $(3,3)$. The agent receives a reward of $-1$ for all state-action pairs before reaching a goal state and $0$ reward in the goal state. In the experiment, the agent first visits $c$, then visits $a$ and $b$ in the order corresponding to the truth value of $c$, and finally remains in the goal state.

\subsection{Multi-Agent System Collision Avoidance with Random Ordered Task and One-Way Lane}\label{exp:5}

In this task, there are two agents $A_1$ and $A_2$, whose locations are known to each other. Agent $A_1$ has one goal location $a$ while agent $A_2$ has two goal locations $b$ and $c$. An agent receives a positive reward only when it is at its goal locations. We require agent $A_2$ to visit both its goal locations, but in a specific order (i.e., $b$ to $c$ or $c$ to $b$). The goal locations are known to both agents but the desired visitation order is known only to agent $A_2$. This order is randomly chosen at the beginning and kept fixed (indicated by the truth value of $o$). We require the agents to avoid collisions between themselves (indicated by the truth value of $col$) and to always stay within the lanes (indicated by the truth value of $s$). This requirements can be specified using $\ltlf$ as $\varphi_5 = (o \to
(\neg b \textbf{U} (c \wedge \textbf{F}b))) \wedge (\neg o \to (\neg c \textbf{U} (b \wedge \textbf{F}c))) \wedge (\textbf{G}(s \wedge \neg col))$.  

\begin{figure}[h]
     \centering
     \begin{subfigure}[b]{0.45\textwidth}
         \centering
         \includegraphics[width=\textwidth]{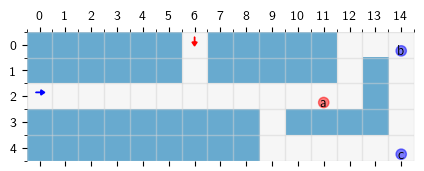}
     \end{subfigure}
        \caption{Multi-agent system collision avoidance benchmark with random ordered tasks, one-way lane, and model $\mathscr{M}_5$.}
        \label{fig:amodel17}
\end{figure}

This task was performed on model $\mathscr{M}_5$ with a larger grid size of $5 \times 15$ and a complex lane structure. The lanes are marked by the grey grid cells in Fig.~\ref{fig:amodel17}. Further, we allow only one-way movement in some lanes. In model  $\mathscr{M}_5$, agent $A_2$ (blue arrow), starting from $(2,0)$, has two goal locations $b$ (0, 14) and $c$ (4, 14), while agent $A_1$ (red arrow), starting from $(0,6)$, has one goal location $a$ (2, 11). Each agent receives a (discounted) reward of $1$  for each time step in their own goal locations. The environment allows only one-way movement in the lane from $(4, 14)$ to $(0, 14)$. Thus, if $o$ indicates that agent $A_2$ has to visit $c$ before $b$, it can do so through the one-way lane between $c$ and $b$. If $o$ indicates the other order of visitation, then agent $A_2$ has to take the longer route from $b$ to $c$ which passes through agent $A_1$'s goal location, i.e., (2, 11). Our result matches this expectation, with the agents moving in a manner so as to avoid collision and maximize the total reward.

\begin{table}
    \centering
    \caption{Reward and constraint performance of the policy $\bar{\mu}$ under various models and specifications.}\label{tab:data}
    \begin{tabular}{lccccc}
      \toprule 
      \bfseries Model & \bfseries Spec & $\mathcal{R}^{\mathscr{M}}(\bar{\mu})$ & $\mathcal{R}^f(\bar{\mu})$ & $1-\delta$ & $B$ \\
      \midrule 
      $\mathscr{M}_1$ & $\varphi_1$& $0.95$ & $0.70$ & $0.70$ & $8$\\ 
      $\mathscr{M}_2$  &$\varphi_2$ & $1.01$ & $0.79$ & $0.80$ & $10$\\ 
      $\mathscr{M}_3$  &$\varphi_1$ & $2.73$ & $0.81$ & $0.85$ & $20$\\
      $\mathscr{M}_4$  &$\varphi_4$ & $-14$ & $0.87$ & $0.9$ & $100$\\
      $\mathscr{M}_5$  &$\varphi_5$ & $1.59$ & $0.72$ & $0.7$ & $50$\\
      \bottomrule 
    \end{tabular}
\end{table}

\section{Conclusions}

We provided a methodology for designing agent policies that maximize the total expected reward while ensuring that the probability of satisfying a linear temporal logic ($\ltlf$) specification is sufficiently high. By augmenting the system state with the state of the DFA associated with the $\ltlf$ specification, we constructed a constrained product POMDP. Solving this constrained product POMDP is equivalent to solving the original problem. We provided an alternative constrained POMDP solver based on the exponentiated gradient (EG) algorithm and derived approximation bounds for it. 
Our methodology is further extended to a multi-agent setting with information asymmetry. For these various settings, we computed near optimal policies that satisfy the $\ltlf$ specification with sufficiently high probability.  We observed in our experiments that our approach results in policies that effectively balance information acquisition (exploration), reward maximization (exploitation) and satisfaction of the specification which is very difficult to achieve using classical POMDPs.

\bibliographystyle{IEEEtran}
\bibliography{references}

\appendices

\section{Proof of Theorem \ref{equiv1}}\label{equiv1proof}
For any policy $\mu$, we have
\allowdisplaybreaks\begin{align*}
    \mathcal{R}^{\mathscr{M}^\times}_{o}(\mu) &= \E_\mu\left[\sum_{t=0}^T r_t^{o\times}(X_t,\U_t)\right]\\
    &= \E_\mu\left[\sum_{t=0}^T r_t^{o\times}((\mathcal{S}_t,Q_t),\U_t)\right]\\
    & \stackrel{a}{=} \E_\mu\left[\sum_{t=0}^T r_t^{o}(\mathcal{S}_t,\U_t)\right] = \mathcal{R}^{\mathscr{M}}_{o}(\mu).
\end{align*}
The equality in $(a)$ follows from the definition of $r_t^{o\times}$ in \eqref{prodrewdefo}. We can similarly show that $\mathcal{R}^{\mathscr{M}^\times}_{c}(\mu) =  \mathcal{R}^{\mathscr{M}}_{c}(\mu)$. Further, using \eqref{finalrewdef}, we have 
\begin{align*}
    r^f(X_{T+1}) &= r^f((\mathcal{S}_{T+1},Q_{T+1}))  = \mathds{1}_F(Q_{T+1}).
\end{align*}

By  construction of the product POMDP dynamics, a run $X_{0:T}$ of the product POMDP satisfies $\varphi$ if and only $Q_{T+1} \in F$. 
Hence,
\begin{align*}
    \mathcal{R}^{f}(\mu) = \E_\mu\left[ r^f(X_{T+1})\right] = \Py_{\varphi}^{\mathscr{M}}(\mu).
\end{align*}

\section{Proof of Lemma \ref{epsopt}}\label{epsoptproof}
By definition of the Lagrangian function for Problem \eqref{prodprobform}, we have:
\allowdisplaybreaks\begin{align}
    \mathcal{R}^* &= l^* \notag \\
    &\leq l^*_B \notag \\
    &\leq \inf_{\boldsymbol{\lambda} \geq 0, \lVert\boldsymbol{\lambda} \rVert_{1}\leq B}L(\bar{\mu},\boldsymbol{\lambda}) + \epsilon \notag \\
    &= \epsilon + \mathcal{R}^{\mathscr{M}^\times}_{o}(\bar{\mu}) + \label{eq:appB_0} \\ 
    &\inf_{\boldsymbol{\lambda}\geq 0,\lVert\boldsymbol{\lambda \rVert}_{1}\leq B}\left[\lambda^f (\mathcal{R}^{f}(\bar{\mu})-1+\delta) + 
    \lambda^c(\mathcal{R}^{\mathscr{M}^\times}_{c}(\bar{\mu}) - \rho) \right]. \notag
\end{align}
There are two possible cases: \\
(i) $\min\{\mathcal{R}^{f}(\bar{\mu})-1+\delta, \mathcal{R}^{\mathscr{M}^\times}_{c}(\bar{\mu}) - \rho\} \geq 0$ and \\(ii) 
$\min\{\mathcal{R}^{f}(\bar{\mu})-1+\delta, \mathcal{R}^{\mathscr{M}^\times}_{c}(\bar{\mu}) - \rho\} < 0$.

If case (i) is true, then  \eqref{conrewsat} and \eqref{consat} are trivially satisfied. Further, in this case,
\begin{align*}
    \inf_{\boldsymbol{\lambda}\geq 0,\lVert\boldsymbol{\lambda \rVert}_{1}\leq B}\left[\lambda^f (\mathcal{R}^{f}(\bar{\mu})-1+\delta) + 
    \lambda^c(\mathcal{R}^{\mathscr{M}^\times}_{c}(\bar{\mu}) - \rho) \right] = 0.
\end{align*}
Therefore, $\mathcal{R}^* \leq \mathcal{R}^{\mathscr{M}^\times}_{o}(\bar{\mu}) + \epsilon$ holds, hence \eqref{objrewsat} is satisfied.

If case (ii) is true, we have
\begin{align}
   &\inf_{\boldsymbol{\lambda}\geq 0,\lVert\boldsymbol{\lambda \rVert}_{1}\leq B}\left[\lambda^f (\mathcal{R}^{f}(\bar{\mu})-1+\delta) + \lambda^c(\mathcal{R}^{\mathscr{M}^\times}_{c}(\bar{\mu}) - \rho) \right] \notag \\
   &= B\min\{\mathcal{R}^{f}(\bar{\mu})-1+\delta, \mathcal{R}^{\mathscr{M}^\times}_{c}(\bar{\mu}) - \rho\} < 0 \label{eq:appB_1} 
\end{align}
The results in \eqref{eq:appB_0} and \eqref{eq:appB_1} together imply that $\mathcal{R}^* \leq \mathcal{R}^{\mathscr{M}^\times}_{o}(\bar{\mu}) + \epsilon$ and hence, \eqref{objrewsat} is satisfied.
Further, by substituting \eqref{eq:appB_1} in \eqref{eq:appB_0} we get
\begin{align*}
    B\min\{\mathcal{R}^{f}(\bar{\mu})-1+\delta, \mathcal{R}^{\mathscr{M}^\times}_{c}(\bar{\mu}) - \rho\} &\geq \mathcal{R}^*- \mathcal{R}^{\mathscr{M}^\times}_{o}(\bar{\mu})-\epsilon\\
    &\geq \mathcal{R}^*- \mathbf{R}^{max}_{o}-\epsilon,
\end{align*}
where the last inequality holds because $\mathbf{R}^{max}_{o}$ is the maximum possible reward. Hence, \eqref{conrewsat} and \eqref{consat} are satisfied.

\section{Proof of Theorem \ref{lagrangethm}}\label{lagrangethmproof}
Consider the dual of \eqref{bsupinf} and define
\begin{align}
    u^*_B:=\inf_{\boldsymbol{\lambda}\geq 0,\lVert \boldsymbol{\lambda} \rVert_{1}\leq B}\sup_{\mu} L(\mu,\boldsymbol{\lambda})\tag{P7}\label{binfsup}.
\end{align}
Then, we obtain $l^*_B \stackrel{a}{\leq} u^*_B$ and 
\allowdisplaybreaks\begin{align}
     u^*_B &= \inf_{\boldsymbol{\lambda}\geq 0,\lVert \boldsymbol{\lambda} \rVert_{1}\leq B}\sup_{\mu} L(\mu,\boldsymbol{\lambda}) \notag\\
    &\stackrel{}{\leq}  \sup_{\mu}L(\mu,\bar{\boldsymbol{\lambda}})\notag\\
    &\stackrel{}{=}  L(\mu_{\bar{\boldsymbol{\lambda}}},\bar{\boldsymbol{\lambda}})\notag\\
    &\stackrel{b}{=} \frac{1}{K}\sum_{k=1}^KL(\mu_{\bar{\boldsymbol{\lambda}}},{\boldsymbol{\lambda}}_k)\notag\\
    &\stackrel{c}{\leq} \frac{1}{K}\sum_{k=1}^KL(\mu_k,{\boldsymbol{\lambda}}_k)\notag\\
    &\stackrel{d}{\leq} \frac{1}{K}\inf_{\boldsymbol{\lambda}\geq 0,\lVert \boldsymbol{\lambda} \rVert_{1}\leq B}\sum_{k=1}^KL(\mu_k,{\boldsymbol{\lambda}})+ 2BG\sqrt{2\log3/K}\notag\\
    &\stackrel{e}{=} \inf_{\boldsymbol{\lambda}\geq 0,\lVert \boldsymbol{\lambda} \rVert_{1}\leq B}L(\bar{\mu},{\boldsymbol{\lambda}})+ 2BG\sqrt{2\log3/K}. \label{eq:appC_1}
\end{align}
The inequality in $(a)$ holds because of weak duality \cite{boyd2004convex}. The equality in $(b)$ holds because $L(\cdot, \cdot)$ is a  bilinear (more precisely, bi-affine) function. The inequality in $(c)$ holds because $\mu_k$ is the maximizer associated with $\boldsymbol{\lambda}_k$. Inequality $(d)$ follows from the no-regret property of the EG algorithm \cite[Corollary 5.7]{hazan2016introduction}. The equality in $(e)$ is again a consequence of the bilinearity of $L(\cdot,\cdot)$. Combining \eqref{eq:appC_1} with Lemma \ref{epsopt} proves the theorem.

\section{Proof of Theorem \ref{lagrangeapproxthm}}\label{lagrangeapproxthmproof}

Define $\hat{L}(\mu,\boldsymbol{\lambda})=\mathcal{R}^{\mathscr{M}^\times}(\mu) + \lambda^{f} (\textsc{eval}(\mu,r^{f}) -1+\delta) + \lambda^{c} (\textsc{eval}(\mu,r^{c\times}) -\rho)$.
Following initial arguments similar to those in the proof of Theorem \ref{lagrangethm}, we obtain:
\allowdisplaybreaks\begin{align}
    l^*_B &\stackrel{}{\leq} \frac{1}{K}\sum_{k=1}^KL(\mu_{\bar{\boldsymbol{\lambda}}},{\boldsymbol{\lambda}}_k) \notag\\
    &\stackrel{a}{\leq} \frac{1}{K}\sum_{k=1}^KL(\mu_{\boldsymbol{\lambda}_k},{\boldsymbol{\lambda}}_k)\notag\\
    &\stackrel{b}{\leq} \frac{1}{K}\sum_{k=1}^KL(\mu_k,{\boldsymbol{\lambda}}_k) + \epsilon_{bp} \notag\\
    &\stackrel{c}{\leq} \frac{1}{K}\sum_{k=1}^K\hat{L}(\mu_k,{\boldsymbol{\lambda}}_k) + \epsilon_{bp} + \epsilon_{est} \notag \\
    &\stackrel{d}{\leq} \frac{1}{K}\inf_{\boldsymbol{\lambda}\geq 0,\lVert \boldsymbol{\lambda} \rVert_{1}\leq B}\sum_{k=1}^K\hat{L}(\mu_k,{\boldsymbol{\lambda}})+ 2BG\sqrt{\frac{2\log3}{K}} + \epsilon_{bp} + \epsilon_{est} \notag\\
    &\stackrel{e}{=} \inf_{\boldsymbol{\lambda}\geq 0,\lVert \boldsymbol{\lambda} \rVert_{1}\leq B}\hat{L}(\bar{\mu},{\boldsymbol{\lambda}})+ 2BG\sqrt{2\log3/K} + \epsilon_{bp} + \epsilon_{est} \notag\\
    &\stackrel{f}{\leq} \inf_{\boldsymbol{\lambda}\geq 0,\lVert \boldsymbol{\lambda} \rVert_{1}\leq B}L(\bar{\mu},{\boldsymbol{\lambda}})+ 2BG\sqrt{2\log3/K} + \epsilon_{bp} + 2\epsilon_{est}. \label{eq:appC_2}
\end{align}
The inequality in $(a)$ holds because $\mu_{\boldsymbol{\lambda}_k}$ is the maximizer associated with $\boldsymbol{\lambda}_k$. The inequality in $(b)$ holds 
by the bounded sub-optimality of $\mu_k$.
Inequality $(c)$ follows from the bounded error of the estimator. Inequality $(d)$ follows from  the no-regret property of the EG algorithm. The equality in $(e)$ is again a consequence of the bilinearity of $\hat{L}(\cdot)$. Finally, inequality $(f)$ follows from the bounded error of the estimator. Combining \eqref{eq:appC_2} with Lemma \ref{epsopt} proves the theorem.
 
\section{Proof of Lemma \ref{disc}}\label{discproof}

\begin{table*}[!htbp]
    \centering
    \caption{Performance Value and Hyper-parameters}\label{tab:add_data}
    \begin{tabular}{lccccccccccccc}
      \toprule 
      \bfseries Model & \bfseries Spec & $|S|$ & $|Q|$ & $\mathcal{R}^{\mathscr{M}}(\bar{\mu})$ & $\mathcal{R}^f(\bar{\mu})$ & $1-\delta$ & $B$ & $\eta$ & $K$ & $simu$ & $T_{solve}$ & $T_{simu}$ & $T_{total}$ \\
      \midrule 
      $\mathscr{M}_1$ & $\varphi_1$&64 &3 &$0.95$ & $0.70$ & $0.70$ & $8$ & $2$ & $50$ & $100$ & $17299$ & $7825$ & $25125$\\ 
      $\mathscr{M}_2$  &$\varphi_2$ &16 &4 & $1.01$ & $0.79$ & $0.80$ & $10$ & $2$ & $100$ & $200$ & $109$ & $718$ & $828$\\ 
      $\mathscr{M}_3$  &$\varphi_1$ &32 &3 & $2.73$ & $0.81$ & $0.85$ & $20$ & $0.02$ & $100$ & $200$ & $370$ & $21676$ & $22046$\\ 
      $\mathscr{M}_4$  &$\varphi_4$ &16 &3 & $-14$ & $0.87$ & $0.9$ & $100$ & $2$ & $20$ & $100$ & $79$ & $14$ & $93$\\ 
      $\mathscr{M}_5$  &$\varphi_5$ &116 &5 & $1.59$ & $0.72$ & $0.7$ & $50$ & $2$ & $50$ & $50$ & $11192$ & $41705$ & $53108$\\ 
      \bottomrule 
    \end{tabular}
\end{table*}

The rewards $\mathcal{R}^{\mathscr{M}^\times}_{o}(\mu)$ and $\mathcal{R}^{f}(\mu)$ in the corresponding product POMDP are given by
\begin{align*}
    \mathcal{R}^{\mathscr{M}^\times}_{o}(\mu) &= \E_\mu\left[\sum_{t=0}^T r_t^{o\times}(X_t,\U_t)\right]\\
    &=(1-\gamma)\E_\mu\left[\sum_{k=0}^{\infty} \gamma^k \sum_{t=0}^k r_t^{o\times}(X_t,\U_t)\right]\\
    &=(1-\gamma)\E_\mu\left[\sum_{t=0}^{\infty}  \sum_{k=0}^{\infty}\gamma^{k+t} r_t^{o\times}(X_t,\U_t)\right]\\
    &= \E_\mu\left[\sum_{t=0}^\infty \gamma^{t}r_t^{o\times}(X_t,\U_t)\right]\\
    \mathcal{R}^{f}(\mu) &= \E_\mu\left[ r^f(X_{T+1})\right]\\
    &= (1-\gamma)\E_\mu\left[\sum_{t=0}^\infty \gamma^{t}r^f(X_{t+1})\right]\\
    &= \frac{(1-\gamma)}{\gamma}\E_\mu\left[\sum_{t=1}^\infty \gamma^{t}r^f(X_{t})\right].
\end{align*}
Similarly, we have $\mathcal{R}^{\mathscr{M}^\times}_{c}(\mu) = \E_\mu\left[\sum_{t=0}^\infty \gamma^{t}r_t^{c\times}(X_t,\U_t)\right]$.
Therefore,
\small{
\begin{align*}
&L(\mu,\boldsymbol{\lambda}) =  -\frac{\lambda^{f}(1-\gamma)}{\gamma}\E[r^f(X_{0})]-\lambda^{f}(1-\delta) - \lambda^{c}\rho  + \\
    &\E_\mu\left[ \sum_{t=0}^\infty \gamma^{t}\left(r_t^{o\times}(X_t,\U_t)+ \frac{\lambda^{f}(1-\gamma)}{\gamma}r^f(X_{t}) + \lambda^{c}r_t^{c\times}(X_t,\U_t)\right) \right].
\end{align*}}

\section{Hyper-parameters and Runtimes}

The parameter $\delta$ in all the experiments is chosen in the following manner: (i) We first solve a POMDP problem aiming to maximize the probability of satisfaction of the $\ltlf$ constraint. Let this probability be denoted by $p_{max}$. (ii) Since any threshold $1-\delta$ larger than $p_{max}$ is infeasible, we choose $\delta$ such that $1-\delta$ is around $0.9p_{max}$. $\eta$ and $B$ are hyperparameters in our experiments. The value of $\eta$ suggested by Theorem \ref{lagrangethm} is guaranteed to result in convergence, but in practice, slightly larger values of $\eta$ can lead to faster convergence.

In Table~\ref{tab:add_data}, we provide additional hyper-parameters that were used in our experiments. The parameter $simu$ denotes the number of Monte-Carlo simulations that were used to estimate the constraint in each iteration. $T_{solve}$ is the total time (over $K$ iterations) spent in solving the unconstrained POMDP using the SARSOP solver \cite{kurniawati2008sarsop}. $T_{simu}$ is the total time spent in simulating policies generated by the SARSOP solver. $T_{total}$ is the overall computation time for that model. 

The runtime (see Table~\ref{tab:add_data}) for our models is drastically different due to differences in three factors: (i) the state size, (ii) the DFA size, and (iii) the complexity of the POMDP problem.

\end{document}